\documentclass[journal]{ieeetran}
\usepackage[pdftex]{hyperref}
\usepackage{siunitx}
\usepackage{enumitem}
\usepackage{amssymb,amsmath,amsthm,graphicx,latexsym,color,multirow}
\usepackage{tikz}
\usepackage{pgfplots}
\usetikzlibrary{calc}
\usepackage{mathtools}

\setlist[description]{leftmargin=0.5cm,labelindent=0cm}

\usepackage[english]{babel}
\usepackage[utf8]{inputenc}
\usepackage{algorithm}
\usepackage[noend]{algpseudocode}

\renewcommand{\thesection}{\arabic{section}}
\renewcommand{\thesubsection}{\arabic{subsection}}
\newtheorem{theorem}{Theorem}[section]

\theoremstyle{definition}
\newtheorem{example}[theorem]{Example}

\newtheorem{definition}[theorem]{Definition}

\newcommand{\norm}[1]{\left\lVert#1\right\rVert_2}

\newcommand{\Traces}{\mathsf{Traces}}

\newcommand{\dparticipation}{\mathsf{participation}}
\newcommand{\dinit}{\mathsf{init}}
\newcommand{\dinitial}{\mathsf{initial}}

\newcommand{\dsucc}{\mathsf{succ}}
\newcommand{\dnotsucc}{\mathsf{notsucc}}
\newcommand{\dresp}{\mathsf{resp}}
\newcommand{\dchainresp}{\mathsf{chainresp}}

\newcommand{\dprec}{\mathsf{prec}}

\newcommand{\dnotcoexist}{\mathsf{not-coexist}}
\newcommand{\dopresp}{\mathsf{opresp}}

\newcommand{\Const}{\mathsf{Const}}

\newcommand{\GETOUT}[1]{}
\definecolor{mygreen}{RGB}{80,160,80}
\definecolor{myonegrey}{RGB}{30,30,30}
\definecolor{mytwogrey}{RGB}{70,70,70}
\definecolor{mythreegrey}{RGB}{110,110,110}
\definecolor{myfourgrey}{RGB}{150,150,150}

\newcommand{\dG}{\mathsf{G}}
\newcommand{\dF}{\mathsf{F}}
\newcommand{\dN}{\mathsf{N}}

\newcommand{\dmodels}{\models}
\newcommand{\dimplies}{\Rightarrow}

\newcommand{\UniqueTraces}{\mathsf{UniqueTraces}}
\newcommand{\GoodTraces}{\mathsf{GoodTraces}}
\newcommand{\Answer}{\mathsf{Optimal}}

\newcommand{\tbox}[1]{\begin{minipage}[t]{1.5cm} \begin{center} \tiny\bf #1\end{center}\end{minipage}}
\newcommand{\zbox}[1]{\begin{minipage}[t]{.5cm} \begin{center} \tiny\bf #1\end{center}\end{minipage}}
\newcommand{\myarr}[3]{\draw [haup] (#1) to node[fill=white] {#3} (#2);}
\newcommand{\myoarr}[3]{\draw [haup] (#1) to node[fill=white] {#3}  (#2);}
\newcommand*{\doi}[1]{\href{http://dx.doi.org/#1}{doi: #1}}

\newcommand{\ADdiagramfigure}{
\begin{figure*}
\begin{center}
\begin{tabular}{|c@{\qquad}|@{\qquad}c|} \hline 
AD1 & AD2 \\ 
\begin{tikzpicture}[outline/.style={draw,thick,fill=black!3, minimum height=0.75cm,minimum width=1.5cm},
                    outline/.default=black,
                    haup/.style={thick,->},
                    hauq/.style={thick,}
					]
\def\myscl{2}
\node [outline] (N1) at (0,6) {\tbox{1. Finish dinner}};
\node [outline] (N2) at (0,4) {\tbox{2. Tidy table}};
\node [outline] (N3) at (0,2) {\tbox{3. Do jigsaw}};
\node [outline] (N4) at (0,0) {\tbox{4. Tidy jigsaw}};
\node [outline] (N5) at (2*\myscl,3) {\tbox{5. Watch bedtime show}};
\node [outline] (N6) at (2*\myscl,0) {\tbox{6. Get ready for bed }};
\node (Nimag) at (2*\myscl,6) {};
\myarr{N1}{N2}{\tiny$\dresp$}
\myarr{N2}{N3}{\tiny$\dprec$}
\myarr{N3}{N4}{\tiny$\dsucc$}
\myarr{N6}{N4}{\tiny$\dnotsucc$}
\myarr{N6}{N5}{\tiny$\dnotsucc$}
\myoarr{Nimag}{N5}{\tiny$\dprec$}
\draw [hauq] (N1) to (2*\myscl,6) to (2*\myscl,5.8);
\node at (0,6.5) {\tiny$\dparticipation$};
\node at (2*\myscl,-0.6) {\phantom{\tiny$\dparticipation$}};
\end{tikzpicture}
&
\begin{tikzpicture}[outline/.style={draw,thick,fill=black!3, minimum height=0.75cm,minimum width=1.5cm},
                    outline/.default=black,
                    haup/.style={thick,->},
                    hauq/.style={thick,}
					]
\def\myscl{2}
\node [outline] (N1) at (0,6) {\tbox{1. Finish dinner}};
\node [outline] (N2) at (0,4) {\tbox{2. Tidy table}};
\node [outline] (N3) at (0,2) {\tbox{3. Do jigsaw}};
\node [outline] (N4) at (0,0) {\tbox{4. Tidy jigsaw}};
\node [outline] (N5) at (2*\myscl,3) {\tbox{5. Watch bedtime show}};
\node [outline] (N6) at (2*\myscl,0) {\tbox{6. Get ready for bed }};
\node (Nimag) at (2*\myscl,6) {};
\myarr{N1}{N2}{\tiny$\dresp$}
\myarr{N2}{N3}{\tiny$\dprec$}
\myarr{N3}{N4}{\tiny$\dsucc$}
\myarr{N6}{N4}{\tiny$\dnotsucc$}
\myarr{N6}{N5}{\tiny$\dnotsucc$}
\myoarr{Nimag}{N5}{\tiny$\dprec$}
\draw [hauq] (N1) to (2*\myscl,6) to (2*\myscl,5.8);
\node at (0,6.5) {\tiny$\dparticipation$};
\node at (2*\myscl,-0.6) {\tiny$\dparticipation$};
\end{tikzpicture} \\ \hline
\end{tabular}
\end{center}
\caption{Illustration of the declarative processes After Dinner 1 and After Dinner 2 in Example~\ref{adk}.\label{adonetwo}}
\end{figure*}
}

\def\PHonediagramfigure{
\begin{figure*}[!h]
\begin{center}
\begin{tikzpicture}[outline/.style={draw,thick,fill=black!3, minimum height=0.75cm,minimum width=1.5cm},
                    outline/.default=black,
                    haup/.style={thick,->},
                    hauq/.style={thick,}
					]
\def\myscl{2}
\node [outline] (N2) at (0,2) {\tbox{2. X-ray}};
\node [outline] (N1) at (0,4) {\tbox{1. Examination}};
\node [outline] (N3) at (4*\myscl,4) {\tbox{3. Surgery}};
\node [outline] (N4) at (2*\myscl,2) {\tbox{4. Fixation}};
\node [outline] (N5) at (0,0) {\tbox{5. Medication}};
\node [outline] (N6) at (4*\myscl,0) {\tbox{6. Cast}};
\node [outline] (N7) at (6*\myscl,4) {\tbox{7. Rehab}};
\node [outline] (N8) at (6*\myscl,2) {\tbox{8. Sling}};
\node [draw,circle,minimum width=.15cm,fill=white] (Cen) at (4*\myscl,2) {\zbox{1of4}};
\node (Nimag) at (2*\myscl,4) {};
\node (Mimag) at (2*\myscl,0) {};
\myarr{N2}{N4}{\tiny$\dprec$}
\myarr{N4}{N6}{\tiny$\dnotcoexist$}
\myarr{N3}{N7}{\tiny$\dopresp$}
\draw [hauq] (N2) to (0,3) to (2*\myscl,3) to (2*\myscl,4) to (2.1*\myscl,4);
\myoarr{Nimag}{N3}{\tiny$\dprec$}
\draw [hauq] (N2) to (0,1) to (2*\myscl,1) to (2*\myscl,0) to (2.1*\myscl,0);
\myoarr{Mimag}{N6}{\tiny$\dprec$}
\node at (0,4.5) {\tiny$\dinit$};
\draw (Cen) to (N3);
\draw (Cen) to (N4);
\draw (Cen) to (N8);
\draw (Cen) to (N6);
\node (Pimag) at (2*\myscl,1.4) {};
\node (Qimag) at (4*\myscl,1.4) {};
\end{tikzpicture}
\caption{Illustration of the declarative process {\it{Patient Handler}} (that we will refer to as PH1)  from ~\cite{vda2009}\label{ph1}}
\end{center}
\end{figure*}
}

\newcommand{\PHtwodiagramfigure}{
\begin{figure*}[!h]
\begin{center}
\begin{tikzpicture}[outline/.style={draw,thick,fill=black!3, minimum height=0.75cm,minimum width=1.5cm},
                    outline/.default=black,
                    haup/.style={thick,->},
                    hauq/.style={thick,}
					]
\def\myscl{2}
\node [outline] (N53) at (0,2) {\tbox{53. Prescr. Strong Painkillers A}};
\node [outline] (N52) at (0,4) {\tbox{52. Prescr. Strong Painkillers B}};
\node [outline] (N51) at (0,6) {\tbox{51. Prescr. Weak Painkillers}};
\node [outline] (N2) at (0,8) {\tbox{2. X-ray}};
\node [outline] (N1) at (0,10) {\tbox{1. Examine Patient}};
\node [outline] (N10) at (1.2*\myscl,0) {\tbox{10. Prescr. anti-caogs}};
\node [outline] (N9) at (3*\myscl,2) {\tbox{9. Prescr. anti-inflams}};
\node [outline] (N3) at (6*\myscl,2) {\tbox{3. Perform surgery}};
\node [outline] (N7) at (6*\myscl,0) {\tbox{7. Let patient rest}};
\node [outline] (N6) at (6*\myscl,10) {\tbox{6. Apply cast}};
\node [outline] (N8) at (4*\myscl,9) {\tbox{8. Apply sling}};
\node [outline] (N4) at (4*\myscl,8) {\tbox{4. Apply fixation}};
\node [outline] (N12) at (4*\myscl,7) {\tbox{12. Apply bandage}};
\node [outline] (N11) at (3*\myscl,-2) {\tbox{11. Perform physio}};
\node [draw,circle,minimum width=.15cm,fill=white] (Cen) at (5.5*\myscl,8) {\zbox{1of5}};
\myarr{N1}{N2}{\tiny$\dprec$}
\myarr{N1}{N6}{\tiny$\dprec$}
\myarr{N3}{N6}{\tiny$\dresp$}
\myarr{N3}{N9}{\tiny$\dresp$}
\myarr{N3}{N7}{\tiny$\dopresp$}
\draw (Cen) to (N6);
\draw (Cen) to (N8);
\draw (Cen) to (N4);
\draw (Cen) to (N12);
\draw (Cen) to (5.85*\myscl,2.4);
\node at (0,10.5) {\tiny$\dinit$};
\draw [hauq] (-0.85,9.8) to (-1.5,9.8) to (-1.5,6) to (N51);
\myoarr{-1.5,9.8}{-1.5,6}{\tiny$\dprec$}
\draw [hauq] (-0.85,10) to (-2.15,10) to (-2.15,4) to (N52);
\myoarr{-2.15,10}{-2.15,4}{\tiny$\dprec$}
\draw [hauq] (-0.85,10.2) to (-2.8,10.2) to (-2.8,2) to (N53);
\myoarr{-2.8,10.2}{-2.8,2}{\tiny$\dprec$}
\draw [hauq] (N53) to (0,0) to (N10);
\node [fill=white] at (0,0) {\tiny$\dnotcoexist$};
\draw [hauq] (N53) to (N9);
\node [fill=white] at (1.8*\myscl,2) {\tiny$\dnotcoexist$};
\draw [hauq] (0.2,9.61) -- (0.2,9.2) -- (1.2*\myscl,9.2) -- (N10);
\myoarr{1.2*\myscl,9.2}{N10}{\tiny$\dprec$}
\draw [hauq] (0.4,9.61) -- (0.4,9.4) -- (1.4*\myscl,9.4) -- (1.4*\myscl,4) -- (3*\myscl,4) -- (N9);
\myoarr{1.4*\myscl,4}{3*\myscl,4}{\tiny$\dprec$}
\draw [hauq] (0.6,9.61) -- (0.6,9.51) -- (1.6*\myscl,9.51) -- (1.6*\myscl,5) -- (4*\myscl,5) -- (4*\myscl,3) -- (5.7*\myscl,3) -- (5.7*\myscl,2.4);
\myoarr{4*\myscl,3}{5.7*\myscl,3}{\tiny$\dopresp$}
\draw [hauq] (0.87,9.7) -- (1.9*\myscl,9.71) -- (1.9*\myscl,8) -- (3.57*\myscl,8);
\myoarr{1.9*\myscl,8}{3.57*\myscl,8}{\tiny$\dprec$}
\draw [hauq] (0.87,9.85) -- (2.2*\myscl,9.85) -- (2.2*\myscl,9) -- (3.57*\myscl,9);
\myoarr{2.2*\myscl,9}{3.57*\myscl,9}{\tiny$\dopresp$}
\draw [hauq] (5.57*\myscl,1.85) -- (5*\myscl,1.85) -- (5*\myscl,0) -- (1.64*\myscl,0);
\myoarr{5*\myscl,0}{1.64*\myscl,0}{\tiny$\dresp$}
\draw[dashed] ($(current bounding box.south east)+(3pt,+40pt)$) rectangle ($(current bounding box.north west)+(-3pt,3pt)$);
\node at (6*\myscl,-1.3) {$D_{PH2a}$};
\draw[dotted] ($(current bounding box.south east)+(6pt,-6pt)$) rectangle ($(current bounding box.north west)+(-6pt,6pt)$);
\node at (6*\myscl,-2.9) {$D_{PH2b}$};
\end{tikzpicture}
\end{center}
\caption{Both versions of Patient Handler 2 are illustrated in the diagram. The constraints for $D_{PH2a}$ are contained within the dashed rectangle, and the constraints for $D_{PH2b}$ are contained within the dotted rectangle. 
\label{ph2diag}}
\end{figure*}
}

\newcommand{\PHonetracesfigure}{
\begin{figure*}[!h]
\begin{center}
\tiny
\renewcommand{\arraystretch}{1}
\begin{tabular}{c@{}l@{}l@{}l@{}l@{}l@{}l@{}c} 
\multicolumn{8}{c}{\normalsize $\UniqueTraces(D_{PH1})$}\\[0.2em] \hline
length & \multicolumn{6}{c}{traces} & number \\ \hline\hline
2 &
$\begin{array}{l}
(1,8)
\end{array}$&
&&&&&1\\ \hline
3 & 
$\begin{array}[t]{l}
(1, 2, 3) \\
(1, 2, 4) 
\end{array}$&
$\begin{array}[t]{l}
(1, 2, 6)
\end{array}$&
$\begin{array}[t]{l}
(1, 5, 8)\\
\end{array}$&
$\begin{array}[t]{l}
(1, 8, 2)\\
\end{array}$&
$\begin{array}[t]{l}
(1, 8, 5)\\
\end{array}$&
$\begin{array}[t]{l}
(1,2,8)
\end{array}$&
 7 \\ \hline
4 &
$\begin{array}[t]{l}
(1, 2, 3, 4)\\
(1, 2, 3, 5)\\
(1, 2, 3, 6)\\
(1, 2, 3, 7)\\
(1, 2, 3, 8)
\end{array}$ &
$\begin{array}[t]{l}
(1, 2, 4, 3)\\
(1, 2, 4, 5)\\
(1, 2, 4, 8)\\
(1, 2, 5, 3)\\
(1, 2, 5, 4)
\end{array}$ &
$\begin{array}[t]{l}
(1, 2, 5, 6)\\
(1, 2, 5, 8)\\
(1, 2, 6, 3)\\
(1, 2, 6, 5)\\
(1, 2, 6, 8)
\end{array}$ &
$\begin{array}[t]{l}
(1, 2, 8, 3)\\
(1, 2, 8, 4)\\
(1, 2, 8, 5)\\
(1, 2, 8, 6)\\
(1, 5, 2, 3)
\end{array}$ &
$\begin{array}[t]{l}
(1, 5, 2, 4)\\
(1, 5, 2, 6)\\
(1, 5, 2, 8)\\
(1, 5, 8, 2)\\
(1, 8, 2, 3)
\end{array}$ &
$\begin{array}[t]{l}
(1, 8, 2, 4)\\
(1, 8, 2, 5)\\
(1, 8, 2, 6)\\
(1, 8, 5, 2)
\end{array}$ & 29 \\ \hline
5 & 
$\begin{array}[t]{l}
(1, 2, 3, 4, 5)\\
(1, 2, 3, 4, 7)\\
(1, 2, 3, 4, 8)\\
(1, 2, 3, 5, 4)\\
(1, 2, 3, 5, 6)\\
(1, 2, 3, 5, 7)\\
(1, 2, 3, 5, 8)\\
(1, 2, 3, 6, 5)\\
(1, 2, 3, 6, 7)\\
(1, 2, 3, 6, 8)\\
(1, 2, 3, 7, 4)\\
(1, 2, 3, 7, 5)\\
(1, 2, 3, 7, 6)\\
(1, 2, 3, 7, 8)
\end{array}$ &
$\begin{array}[t]{l}
(1, 2, 3, 8, 4)\\
(1, 2, 3, 8, 5)\\
(1, 2, 3, 8, 6)\\
(1, 2, 3, 8, 7)\\
(1, 2, 4, 3, 5)\\
(1, 2, 4, 3, 7)\\
(1, 2, 4, 3, 8)\\
(1, 2, 4, 5, 3)\\
(1, 2, 4, 5, 8)\\
(1, 2, 4, 8, 3)\\
(1, 2, 4, 8, 5)\\
(1, 2, 5, 3, 4)\\
(1, 2, 5, 3, 6)\\
(1, 2, 5, 3, 7)
\end{array}$ &
$\begin{array}[t]{l}
(1, 2, 5, 3, 8)\\
(1, 2, 5, 4, 3)\\
(1, 2, 5, 4, 8)\\
(1, 2, 5, 6, 3)\\
(1, 2, 5, 6, 8)\\
(1, 2, 5, 8, 3)\\
(1, 2, 5, 8, 4)\\
(1, 2, 5, 8, 6)\\
(1, 2, 6, 3, 5)\\
(1, 2, 6, 3, 7)\\
(1, 2, 6, 3, 8)\\
(1, 2, 6, 5, 3)\\
(1, 2, 6, 5, 8)\\
(1, 2, 6, 8, 3)
\end{array}$ &
$\begin{array}[t]{l}
(1, 2, 6, 8, 5)\\
(1, 2, 8, 3, 4)\\
(1, 2, 8, 3, 5)\\
(1, 2, 8, 3, 6)\\
(1, 2, 8, 3, 7)\\
(1, 2, 8, 4, 3)\\
(1, 2, 8, 4, 5)\\
(1, 2, 8, 5, 3)\\
(1, 2, 8, 5, 4)\\
(1, 2, 8, 5, 6)\\
(1, 2, 8, 6, 3)\\
(1, 2, 8, 6, 5)\\
(1, 5, 2, 3, 4)\\
(1, 5, 2, 3, 6)
\end{array}$ &
$\begin{array}[t]{l}
(1, 5, 2, 3, 7)\\
(1, 5, 2, 3, 8)\\
(1, 5, 2, 4, 3)\\
(1, 5, 2, 4, 8)\\
(1, 5, 2, 6, 3)\\
(1, 5, 2, 6, 8)\\
(1, 5, 2, 8, 3)\\
(1, 5, 2, 8, 4)\\
(1, 5, 2, 8, 6)\\
(1, 5, 8, 2, 3)\\
(1, 5, 8, 2, 4)\\
(1, 5, 8, 2, 6)\\
(1, 8, 2, 3, 4)\\
(1, 8, 2, 3, 5)
\end{array}$ &
$\begin{array}[t]{l}
(1, 8, 2, 3, 6)\\
(1, 8, 2, 3, 7)\\
(1, 8, 2, 4, 3)\\
(1, 8, 2, 4, 5)\\
(1, 8, 2, 5, 3)\\
(1, 8, 2, 5, 4)\\
(1, 8, 2, 5, 6)\\
(1, 8, 2, 6, 3)\\
(1, 8, 2, 6, 5)\\
(1, 8, 5, 2, 3)\\
(1, 8, 5, 2, 4)\\
(1, 8, 5, 2, 6)
\end{array}$ & 82 \\ \hline
6 & 
$\begin{array}[t]{l}
(1, 2, 3, 4, 5, 7)\\
(1, 2, 3, 4, 5, 8)\\
(1, 2, 3, 4, 7, 5)\\
(1, 2, 3, 4, 7, 8)\\
(1, 2, 3, 4, 8, 5)\\
(1, 2, 3, 4, 8, 7)\\
(1, 2, 3, 5, 4, 7)\\
(1, 2, 3, 5, 4, 8)\\
(1, 2, 3, 5, 6, 7)\\
(1, 2, 3, 5, 6, 8)\\
(1, 2, 3, 5, 7, 4)\\
(1, 2, 3, 5, 7, 6)\\
(1, 2, 3, 5, 7, 8)\\
(1, 2, 3, 5, 8, 4)\\
(1, 2, 3, 5, 8, 6)\\
(1, 2, 3, 5, 8, 7)\\
(1, 2, 3, 6, 5, 7)\\
(1, 2, 3, 6, 5, 8)\\
(1, 2, 3, 6, 7, 5)\\
(1, 2, 3, 6, 7, 8)\\
(1, 2, 3, 6, 8, 5)\\
(1, 2, 3, 6, 8, 7)\\
(1, 2, 3, 7, 4, 5)\\
(1, 2, 3, 7, 4, 8)\\
(1, 2, 3, 7, 5, 4)\\
(1, 2, 3, 7, 5, 6)\\
(1, 2, 3, 7, 5, 8)
\end{array}$
&
$\begin{array}[t]{l}
(1, 2, 3, 7, 6, 5)\\
(1, 2, 3, 7, 6, 8)\\
(1, 2, 3, 7, 8, 4)\\
(1, 2, 3, 7, 8, 5)\\
(1, 2, 3, 7, 8, 6)\\
(1, 2, 3, 8, 4, 5)\\
(1, 2, 3, 8, 4, 7)\\
(1, 2, 3, 8, 5, 4)\\
(1, 2, 3, 8, 5, 6)\\
(1, 2, 3, 8, 5, 7)\\
(1, 2, 3, 8, 6, 5)\\
(1, 2, 3, 8, 6, 7)\\
(1, 2, 3, 8, 7, 4)\\
(1, 2, 3, 8, 7, 5)\\
(1, 2, 3, 8, 7, 6)\\
(1, 2, 4, 3, 5, 7)\\
(1, 2, 4, 3, 5, 8)\\
(1, 2, 4, 3, 7, 5)\\
(1, 2, 4, 3, 7, 8)\\
(1, 2, 4, 3, 8, 5)\\
(1, 2, 4, 3, 8, 7)\\
(1, 2, 4, 5, 3, 7)\\
(1, 2, 4, 5, 3, 8)\\
(1, 2, 4, 5, 8, 3)\\
(1, 2, 4, 8, 3, 5)\\
(1, 2, 4, 8, 3, 7)\\
(1, 2, 4, 8, 5, 3)
\end{array}$
&
$\begin{array}[t]{l}
(1, 2, 5, 3, 4, 7)\\
(1, 2, 5, 3, 4, 8)\\
(1, 2, 5, 3, 6, 7)\\
(1, 2, 5, 3, 6, 8)\\
(1, 2, 5, 3, 7, 4)\\
(1, 2, 5, 3, 7, 6)\\
(1, 2, 5, 3, 7, 8)\\
(1, 2, 5, 3, 8, 4)\\
(1, 2, 5, 3, 8, 6)\\
(1, 2, 5, 3, 8, 7)\\
(1, 2, 5, 4, 3, 7)\\
(1, 2, 5, 4, 3, 8)\\
(1, 2, 5, 4, 8, 3)\\
(1, 2, 5, 6, 3, 7)\\
(1, 2, 5, 6, 3, 8)\\
(1, 2, 5, 6, 8, 3)\\
(1, 2, 5, 8, 3, 4)\\
(1, 2, 5, 8, 3, 6)\\
(1, 2, 5, 8, 3, 7)\\
(1, 2, 5, 8, 4, 3)\\
(1, 2, 5, 8, 6, 3)\\
(1, 2, 6, 3, 5, 7)\\
(1, 2, 6, 3, 5, 8)\\
(1, 2, 6, 3, 7, 5)\\
(1, 2, 6, 3, 7, 8)\\
(1, 2, 6, 3, 8, 5)\\
(1, 2, 6, 3, 8, 7)
\end{array}$
&
$\begin{array}[t]{l}
(1, 2, 6, 5, 3, 7)\\
(1, 2, 6, 5, 3, 8)\\
(1, 2, 6, 5, 8, 3)\\
(1, 2, 6, 8, 3, 5)\\
(1, 2, 6, 8, 3, 7)\\
(1, 2, 6, 8, 5, 3)\\
(1, 2, 8, 3, 4, 5)\\
(1, 2, 8, 3, 4, 7)\\
(1, 2, 8, 3, 5, 4)\\
(1, 2, 8, 3, 5, 6)\\
(1, 2, 8, 3, 5, 7)\\
(1, 2, 8, 3, 6, 5)\\
(1, 2, 8, 3, 6, 7)\\
(1, 2, 8, 3, 7, 4)\\
(1, 2, 8, 3, 7, 5)\\
(1, 2, 8, 3, 7, 6)\\
(1, 2, 8, 4, 3, 5)\\
(1, 2, 8, 4, 3, 7)\\
(1, 2, 8, 4, 5, 3)\\
(1, 2, 8, 5, 3, 4)\\
(1, 2, 8, 5, 3, 6)\\
(1, 2, 8, 5, 3, 7)\\
(1, 2, 8, 5, 4, 3)\\
(1, 2, 8, 5, 6, 3)\\
(1, 2, 8, 6, 3, 5)\\
(1, 2, 8, 6, 3, 7)\\
(1, 2, 8, 6, 5, 3)
\end{array}$ &
$\begin{array}[t]{l}
(1, 5, 2, 3, 4, 7)\\
(1, 5, 2, 3, 4, 8)\\
(1, 5, 2, 3, 6, 7)\\
(1, 5, 2, 3, 6, 8)\\
(1, 5, 2, 3, 7, 4)\\
(1, 5, 2, 3, 7, 6)\\
(1, 5, 2, 3, 7, 8)\\
(1, 5, 2, 3, 8, 4)\\
(1, 5, 2, 3, 8, 6)\\
(1, 5, 2, 3, 8, 7)\\
(1, 5, 2, 4, 3, 7)\\
(1, 5, 2, 4, 3, 8)\\
(1, 5, 2, 4, 8, 3)\\
(1, 5, 2, 6, 3, 7)\\
(1, 5, 2, 6, 3, 8)\\
(1, 5, 2, 6, 8, 3)\\
(1, 5, 2, 8, 3, 4)\\
(1, 5, 2, 8, 3, 6)\\
(1, 5, 2, 8, 3, 7)\\
(1, 5, 2, 8, 4, 3)\\
(1, 5, 2, 8, 6, 3)\\
(1, 5, 8, 2, 3, 4)\\
(1, 5, 8, 2, 3, 6)\\
(1, 5, 8, 2, 3, 7)\\
(1, 5, 8, 2, 4, 3)\\
(1, 5, 8, 2, 6, 3)\\
(1, 8, 2, 3, 4, 5)
\end{array}$ &
$\begin{array}[t]{l}
(1, 8, 2, 3, 4, 7)\\
(1, 8, 2, 3, 5, 4)\\
(1, 8, 2, 3, 5, 6)\\
(1, 8, 2, 3, 5, 7)\\
(1, 8, 2, 3, 6, 5)\\
(1, 8, 2, 3, 6, 7)\\
(1, 8, 2, 3, 7, 4)\\
(1, 8, 2, 3, 7, 5)\\
(1, 8, 2, 3, 7, 6)\\
(1, 8, 2, 4, 3, 5)\\
(1, 8, 2, 4, 3, 7)\\
(1, 8, 2, 4, 5, 3)\\
(1, 8, 2, 5, 3, 4)\\
(1, 8, 2, 5, 3, 6)\\
(1, 8, 2, 5, 3, 7)\\
(1, 8, 2, 5, 4, 3)\\
(1, 8, 2, 5, 6, 3)\\
(1, 8, 2, 6, 3, 5)\\
(1, 8, 2, 6, 3, 7)\\
(1, 8, 2, 6, 5, 3)\\
(1, 8, 5, 2, 3, 4)\\
(1, 8, 5, 2, 3, 6)\\
(1, 8, 5, 2, 3, 7)\\
(1, 8, 5, 2, 4, 3)\\
(1, 8, 5, 2, 6, 3)
\end{array}$ & 160 \\ \hline
7 & 
$\begin{array}[t]{l}
(1, 2, 3, 4, 5, 7, 8) \\
(1, 2, 3, 4, 5, 8, 7) \\
(1, 2, 3, 4, 7, 5, 8) \\
(1, 2, 3, 4, 7, 8, 5) \\
(1, 2, 3, 4, 8, 5, 7) \\
(1, 2, 3, 4, 8, 7, 5) \\
(1, 2, 3, 5, 4, 7, 8) \\
(1, 2, 3, 5, 4, 8, 7) \\
(1, 2, 3, 5, 6, 7, 8) \\
(1, 2, 3, 5, 6, 8, 7) \\
(1, 2, 3, 5, 7, 4, 8) \\
(1, 2, 3, 5, 7, 6, 8) \\
(1, 2, 3, 5, 7, 8, 4) \\
(1, 2, 3, 5, 7, 8, 6) \\
(1, 2, 3, 5, 8, 4, 7) \\
(1, 2, 3, 5, 8, 6, 7) \\
(1, 2, 3, 5, 8, 7, 4) \\
(1, 2, 3, 5, 8, 7, 6) \\
(1, 2, 3, 6, 5, 7, 8) \\
(1, 2, 3, 6, 5, 8, 7) \\
(1, 2, 3, 6, 7, 5, 8) \\
(1, 2, 3, 6, 7, 8, 5) \\
(1, 2, 3, 6, 8, 5, 7) \\
(1, 2, 3, 6, 8, 7, 5) \\
(1, 2, 3, 7, 4, 5, 8) \\
(1, 2, 3, 7, 4, 8, 5) \\
(1, 2, 3, 7, 5, 4, 8) \\
(1, 2, 3, 7, 5, 6, 8) \\
(1, 2, 3, 7, 5, 8, 4) \\
(1, 2, 3, 7, 5, 8, 6)
\end{array}$ &
$\begin{array}[t]{l}
(1, 2, 3, 7, 6, 5, 8) \\
(1, 2, 3, 7, 6, 8, 5) \\
(1, 2, 3, 7, 8, 4, 5) \\
(1, 2, 3, 7, 8, 5, 4) \\
(1, 2, 3, 7, 8, 5, 6) \\
(1, 2, 3, 7, 8, 6, 5) \\
(1, 2, 3, 8, 4, 5, 7) \\
(1, 2, 3, 8, 4, 7, 5) \\
(1, 2, 3, 8, 5, 4, 7) \\
(1, 2, 3, 8, 5, 6, 7) \\
(1, 2, 3, 8, 5, 7, 4) \\
(1, 2, 3, 8, 5, 7, 6) \\
(1, 2, 3, 8, 6, 5, 7) \\
(1, 2, 3, 8, 6, 7, 5) \\
(1, 2, 3, 8, 7, 4, 5) \\
(1, 2, 3, 8, 7, 5, 4) \\
(1, 2, 3, 8, 7, 5, 6) \\
(1, 2, 3, 8, 7, 6, 5) \\
(1, 2, 4, 3, 5, 7, 8) \\
(1, 2, 4, 3, 5, 8, 7) \\
(1, 2, 4, 3, 7, 5, 8) \\
(1, 2, 4, 3, 7, 8, 5) \\
(1, 2, 4, 3, 8, 5, 7) \\
(1, 2, 4, 3, 8, 7, 5) \\
(1, 2, 4, 5, 3, 7, 8) \\
(1, 2, 4, 5, 3, 8, 7) \\
(1, 2, 4, 5, 8, 3, 7) \\
(1, 2, 4, 8, 3, 5, 7) \\
(1, 2, 4, 8, 3, 7, 5) \\
(1, 2, 4, 8, 5, 3, 7)
\end{array}$ &
$\begin{array}[t]{l}
(1, 2, 5, 3, 4, 7, 8) \\
(1, 2, 5, 3, 4, 8, 7) \\
(1, 2, 5, 3, 6, 7, 8) \\
(1, 2, 5, 3, 6, 8, 7) \\
(1, 2, 5, 3, 7, 4, 8) \\
(1, 2, 5, 3, 7, 6, 8) \\
(1, 2, 5, 3, 7, 8, 4) \\
(1, 2, 5, 3, 7, 8, 6) \\
(1, 2, 5, 3, 8, 4, 7) \\
(1, 2, 5, 3, 8, 6, 7) \\
(1, 2, 5, 3, 8, 7, 4) \\
(1, 2, 5, 3, 8, 7, 6) \\
(1, 2, 5, 4, 3, 7, 8) \\
(1, 2, 5, 4, 3, 8, 7) \\
(1, 2, 5, 4, 8, 3, 7) \\
(1, 2, 5, 6, 3, 7, 8) \\
(1, 2, 5, 6, 3, 8, 7) \\
(1, 2, 5, 6, 8, 3, 7) \\
(1, 2, 5, 8, 3, 4, 7) \\
(1, 2, 5, 8, 3, 6, 7) \\
(1, 2, 5, 8, 3, 7, 4) \\
(1, 2, 5, 8, 3, 7, 6) \\
(1, 2, 5, 8, 4, 3, 7) \\
(1, 2, 5, 8, 6, 3, 7) \\
(1, 2, 6, 3, 5, 7, 8) \\
(1, 2, 6, 3, 5, 8, 7) \\
(1, 2, 6, 3, 7, 5, 8) \\
(1, 2, 6, 3, 7, 8, 5) \\
(1, 2, 6, 3, 8, 5, 7) \\
(1, 2, 6, 3, 8, 7, 5)
\end{array}$ &
$\begin{array}[t]{l}
(1, 2, 6, 5, 3, 7, 8) \\
(1, 2, 6, 5, 3, 8, 7) \\
(1, 2, 6, 5, 8, 3, 7) \\
(1, 2, 6, 8, 3, 5, 7) \\
(1, 2, 6, 8, 3, 7, 5) \\
(1, 2, 6, 8, 5, 3, 7) \\
(1, 2, 8, 3, 4, 5, 7) \\
(1, 2, 8, 3, 4, 7, 5) \\
(1, 2, 8, 3, 5, 4, 7) \\
(1, 2, 8, 3, 5, 6, 7) \\
(1, 2, 8, 3, 5, 7, 4) \\
(1, 2, 8, 3, 5, 7, 6) \\
(1, 2, 8, 3, 6, 5, 7) \\
(1, 2, 8, 3, 6, 7, 5) \\
(1, 2, 8, 3, 7, 4, 5) \\
(1, 2, 8, 3, 7, 5, 4) \\
(1, 2, 8, 3, 7, 5, 6) \\
(1, 2, 8, 3, 7, 6, 5) \\
(1, 2, 8, 4, 3, 5, 7) \\
(1, 2, 8, 4, 3, 7, 5) \\
(1, 2, 8, 4, 5, 3, 7) \\
(1, 2, 8, 5, 3, 4, 7) \\
(1, 2, 8, 5, 3, 6, 7) \\
(1, 2, 8, 5, 3, 7, 4) \\
(1, 2, 8, 5, 3, 7, 6) \\
(1, 2, 8, 5, 4, 3, 7) \\
(1, 2, 8, 5, 6, 3, 7) \\
(1, 2, 8, 6, 3, 5, 7) \\
(1, 2, 8, 6, 3, 7, 5) \\
(1, 2, 8, 6, 5, 3, 7)
\end{array}$ &
$\begin{array}[t]{l}
(1, 5, 2, 3, 4, 7, 8) \\
(1, 5, 2, 3, 4, 8, 7) \\
(1, 5, 2, 3, 6, 7, 8) \\
(1, 5, 2, 3, 6, 8, 7) \\
(1, 5, 2, 3, 7, 4, 8) \\
(1, 5, 2, 3, 7, 6, 8) \\
(1, 5, 2, 3, 7, 8, 4) \\
(1, 5, 2, 3, 7, 8, 6) \\
(1, 5, 2, 3, 8, 4, 7) \\
(1, 5, 2, 3, 8, 6, 7) \\
(1, 5, 2, 3, 8, 7, 4) \\
(1, 5, 2, 3, 8, 7, 6) \\
(1, 5, 2, 4, 3, 7, 8) \\
(1, 5, 2, 4, 3, 8, 7) \\
(1, 5, 2, 4, 8, 3, 7) \\
(1, 5, 2, 6, 3, 7, 8) \\
(1, 5, 2, 6, 3, 8, 7) \\
(1, 5, 2, 6, 8, 3, 7) \\
(1, 5, 2, 8, 3, 4, 7) \\
(1, 5, 2, 8, 3, 6, 7) \\
(1, 5, 2, 8, 3, 7, 4) \\
(1, 5, 2, 8, 3, 7, 6) \\
(1, 5, 2, 8, 4, 3, 7) \\
(1, 5, 2, 8, 6, 3, 7) \\
(1, 5, 8, 2, 3, 4, 7) \\
(1, 5, 8, 2, 3, 6, 7) \\
(1, 5, 8, 2, 3, 7, 4) \\
(1, 5, 8, 2, 3, 7, 6) \\
(1, 5, 8, 2, 4, 3, 7) \\
(1, 5, 8, 2, 6, 3, 7)
\end{array}$ &
$\begin{array}[t]{l}
(1, 8, 2, 3, 4, 5, 7) \\
(1, 8, 2, 3, 4, 7, 5) \\
(1, 8, 2, 3, 5, 4, 7) \\
(1, 8, 2, 3, 5, 6, 7) \\
(1, 8, 2, 3, 5, 7, 4) \\
(1, 8, 2, 3, 5, 7, 6) \\
(1, 8, 2, 3, 6, 5, 7) \\
(1, 8, 2, 3, 6, 7, 5) \\
(1, 8, 2, 3, 7, 4, 5) \\
(1, 8, 2, 3, 7, 5, 4) \\
(1, 8, 2, 3, 7, 5, 6) \\
(1, 8, 2, 3, 7, 6, 5) \\
(1, 8, 2, 4, 3, 5, 7) \\
(1, 8, 2, 4, 3, 7, 5) \\
(1, 8, 2, 4, 5, 3, 7) \\
(1, 8, 2, 5, 3, 4, 7) \\
(1, 8, 2, 5, 3, 6, 7) \\
(1, 8, 2, 5, 3, 7, 4) \\
(1, 8, 2, 5, 3, 7, 6) \\
(1, 8, 2, 5, 4, 3, 7) \\
(1, 8, 2, 5, 6, 3, 7) \\
(1, 8, 2, 6, 3, 5, 7) \\
(1, 8, 2, 6, 3, 7, 5) \\
(1, 8, 2, 6, 5, 3, 7) \\
(1, 8, 5, 2, 3, 4, 7) \\
(1, 8, 5, 2, 3, 6, 7) \\
(1, 8, 5, 2, 3, 7, 4) \\
(1, 8, 5, 2, 3, 7, 6) \\
(1, 8, 5, 2, 4, 3, 7) \\
(1, 8, 5, 2, 6, 3, 7)
\end{array}$ &
180 \\ \hline
\end{tabular}
\caption{The unique traces for the Patient Handler 1 example $D_{PH1}$\label{ph1traces}}
\end{center}
\end{figure*}
\end{example}
}

\newcommand{\PHtwoatracesfigure}{
\begin{figure*}[!h]
\begin{center}
\footnotesize
\renewcommand{\arraystretch}{1.05}
\begin{tabular}{cllllllc} 
\multicolumn{8}{c}{\normalsize $\UniqueTraces(D_{PH2a})$ summary}\\ \hline
length & \multicolumn{6}{c}{traces} & number \\ \hline\hline
2 & 
$\begin{array}{l}
(1, 4)
\end{array}$&
$\begin{array}{l}
(1, 6)
\end{array}$&
$\begin{array}{l}
(1, 8)
\end{array}$&
$\begin{array}{l}
(1,12)
\end{array}$&&& 4 \\ \hline
3 &
$\begin{array}[t]{l}
(1, 4, 6)\\
(1, 6, 4)\\
(1, 8, 4)\\
(1, 4, 8)\\
(1, 9, 4)\\
(1, 4, 9)\\
(1, 10, 4)\\
(1, 4, 10)\\
(1, 12, 4)\\
(1, 4, 12)\\
(1, 4, 53)\\
(1, 53, 4)
\end{array}$ &
$\begin{array}[t]{l}
(1, 8, 6)\\
(1, 6, 8)\\
(1, 9, 6)\\
(1, 6, 9)\\
(1, 10, 6)\\
(1, 6, 10)\\
(1, 12, 6)\\
(1, 6, 12)\\
(1, 53, 6)\\
(1, 6, 53)\\
(1, 8, 9)\\
(1, 9, 8)
\end{array}$
&
$\begin{array}[t]{l}
(1, 8, 10)\\
(1, 10, 8)\\
(1, 8, 12)\\
(1, 12, 8)\\
(1, 8, 53)\\
(1, 53, 8)\\
(1, 9, 12)\\
(1, 12, 9)\\
(1, 10, 12)\\
(1, 12, 10)\\
(1, 12, 53)\\
(1, 53, 12)
\end{array}$ &
$\begin{array}[t]{l}
(1, 2, 4)\\
(1, 4, 2)\\
(1, 2, 6)\\
(1, 6, 2)\\
(1, 2, 8)\\
(1, 8, 2)\\
(1, 2, 12)\\
(1, 12, 2)\\
(1, 51, 4)\\
(1, 4, 51)\\
(1, 51, 6)\\
(1, 6, 51)
\end{array}$
&
$\begin{array}[t]{l}
(1, 51, 8)\\
(1, 8, 51)\\
(1, 51, 12)\\
(1, 12, 51)\\
(1, 52, 4)\\
(1, 4, 52)\\
(1, 52, 6)\\
(1, 6, 52)\\
(1, 52, 8)\\
(1, 8, 52)\\
(1, 52, 12)\\
(1, 12, 52)
\end{array}$
&&
60 \\ \hline
4 & \multicolumn{6}{l}{
$(1, 8, 4, 6), ~ \ldots$}
& 552 \\ \hline
5 & \multicolumn{6}{l}{
$(1,3,9,10,6), ~ \ldots$}
&  3726 \\ \hline
6 & \multicolumn{6}{l}{
$(1,3,9,10,4,6), ~ \ldots $}
&  19404 \\ \hline
7 & \multicolumn{6}{l}{
$(1,3,4,6,7,9,10), ~ \ldots $} 
& 79164 \\ \hline
8 & \multicolumn{6}{l}{
$(1,3,4,6,7,8,9,10), ~\ldots$}
& 257040 \\ \hline
9  & \multicolumn{6}{l}{
$(1,3,4,6,7,8,9,10,12), ~ \ldots$}
&  715680 \\ \hline
10 & \multicolumn{6}{l}{
$(1,2,3,4,6,7,8,9,19,12),  ~ \ldots $}
& 1995840 \\ \hline
11 & \multicolumn{6}{l}{
$(1,51,2,3,4,6,7,8,9,10,12), ~ \ldots $}
& 5261760 \\ \hline
12 & \multicolumn{6}{l}{
$(1,52,51,2,3,4,6,7,8,9,10,12), ~ \ldots$}
& 7983360 \\ \hline
\end{tabular}
\caption{Summary of the unique traces for Patient Handler $D_{PH2a}$\label{ph2atraces}}
\end{center}
\end{figure*}
}

\newcommand{\PHtwobtracesfigure}{
\begin{figure*}[!h]
\begin{center}
\footnotesize
\renewcommand{\arraystretch}{1.05}
\begin{tabular}{cllllllc} 
\multicolumn{8}{c}{\normalsize $\UniqueTraces(D_{PH2b})$ summary}\\ \hline
length & \multicolumn{6}{c}{traces} & number \\ \hline\hline
2 & 
$\begin{array}{l}
(1, 4)
\end{array}$&
$\begin{array}{l}
(1, 6)
\end{array}$&
$\begin{array}{l}
(1, 8)
\end{array}$&
$\begin{array}{l}
(1,12)
\end{array}$&&& 4 \\ \hline
3 &
$\begin{array}[t]{l}
(1, 11, 4)\\
(1, 4, 11)\\
(1, 11, 6)\\
(1, 6, 11)\\
(1, 11, 8)\\
(1, 8, 11)\\
(1, 11, 12)\\
(1, 12, 11)\\
(1, 4, 6)\\
(1, 6, 4)\\
(1, 8, 4)\\
(1, 4, 8)\\
(1, 9, 4)\\
(1, 4, 9)
\end{array}$ &
$\begin{array}[t]{l}
(1, 10, 4)\\
(1, 4, 10)\\
(1, 12, 4)\\
(1, 4, 12)\\
(1, 4, 53)\\
(1, 53, 4)\\
(1, 8, 6)\\
(1, 6, 8)\\
(1, 9, 6)\\
(1, 6, 9)\\
(1, 10, 6)\\
(1, 6, 10)\\
(1, 12, 6)\\
(1, 6, 12)
\end{array}$
&
$\begin{array}[t]{l}
(1, 53, 6)\\
(1, 6, 53)\\
(1, 8, 9)\\
(1, 9, 8)\\
(1, 8, 10)\\
(1, 10, 8)\\
(1, 8, 12)\\
(1, 12, 8)\\
(1, 8, 53)\\
(1, 53, 8)\\
(1, 9, 12)\\
(1, 12, 9)\\
(1, 10, 12)\\
(1, 12, 10)
\end{array}$ &
$\begin{array}[t]{l}
(1, 12, 53)\\
(1, 53, 12)\\
(1, 2, 4)\\
(1, 4, 2)\\
(1, 2, 6)\\
(1, 6, 2)\\
(1, 2, 8)\\
(1, 8, 2)\\
(1, 2, 12)\\
(1, 12, 2)\\
(1, 51, 4)\\
(1, 4, 51)\\
(1, 51, 6)\\
(1, 6, 51)
\end{array}$
&
$\begin{array}[t]{l}
(1, 51, 8)\\
(1, 8, 51)\\
(1, 51, 12)\\
(1, 12, 51)\\
(1, 52, 4)\\
(1, 4, 52)\\
(1, 52, 6)\\
(1, 6, 52)\\
(1, 52, 8)\\
(1, 8, 52)\\
(1, 52, 12)\\
(1, 12, 52)
\end{array}$
&&
68 \\ \hline
4 & 
\multicolumn{6}{l}{$(1, 11, 4, 6), ~
(1, 4, 11, 6), ~ \ldots$}
& 732 \\ \hline
5 & \multicolumn{6}{l}{
$(1, 11, 8, 4, 6), ~
(1, 8, 11, 4, 6), ~ \ldots$}
& 5934 \\ \hline
6 & \multicolumn{6}{l}{
$(1,11,3,9,10,6), ~
(1,3,11,9,10,6), ~ \ldots$}
& 38034 \\ \hline
7 & \multicolumn{6}{l}{
$(1,11,3,9,10,4,6), ~
(1,3,11,9,10,4,6), ~ \ldots $}
& 195588\\ \hline
8 & \multicolumn{6}{l}{
$(1,11,3,4,6,7,9,10), ~ 
(1,3,11,4,6,7,9,10), ~ \ldots $} 
& 811188\\ \hline
9 & \multicolumn{6}{l}{
$(1,11,3,4,6,7,8,9,10), ~
(1,3,11,4,6,7,8,9,10), ~\ldots$}
& 2772000\\ \hline
10  & \multicolumn{6}{l}{
$(1,11,3,4,6,7,8,9,10,12), ~
(1,3,11,4,6,7,8,9,10,12) , ~ \ldots$}
& 8436960 \\ \hline
11 & \multicolumn{6}{l}{
$(1,11,2,3,4,6,7,8,9,19,12), ~
(1,2,11,3,4,6,7,8,9,19,12), ~ \ldots $}
& 25220160\\ \hline
12 & \multicolumn{6}{l}{
$(1,11,51,2,3,4,6,7,8,9,10,12), ~
(1,51,11,2,3,4,6,7,8,9,10,12), ~ \ldots $}
& 65862720 \\ \hline
13 & \multicolumn{6}{l}{
$(1,11,52,51,2,3,4,6,7,8,9,10,12), ~
(1,52,11,51,2,3,4,6,7,8,9,10,12), ~ \ldots$}
& 95800320 \\ \hline
\end{tabular}
\caption{Summary of the unique traces for Patient Handler $D_{PH2b}$\label{ph2btraces}}
\end{center}
\end{figure*}
}

\begin{document}
\title{Stakeholder utility measures for declarative processes and their use in process comparisons
\thanks{IEEE Transactions on Computational Social Systems, \doi{10.1109/TCSS.2021.3092285}.}\thanks{This version corrects some minor calculation errors in the published article.}}
\author{Mark Dukes \thanks{School of Mathematics and Statistics, University College Dublin, Dublin 4, Ireland.}}
\maketitle
\begin{abstract}
We present a method for calculating and analyzing stakeholder utilities of processes that arise in, but are not limited to, the social sciences. These areas include business process analysis, healthcare workflow analysis and policy process analysis. This method is quite general and applicable to any situation in which declarative-type constraints of a modal and/or temporal nature play a part.

A declarative process is a process in which activities may freely happen while respecting a set of constraints. For such a process, anything may happen so long as it is not explicitly forbidden. Declarative processes have been used and studied as models of business and healthcare workflows by several authors. In considering a declarative process as a model of some system it is natural to consider how the process behaves with respect to stakeholders. We derive a measure for stakeholder utility that can be applied in a very general setting. This derivation is achieved by listing a collection a properties which we argue such a stakeholder utility function ought to satisfy, and then using these to show a very specific form must hold for such a utility. The utility measure depends on the set of unique traces of the declarative process, and calculating this set requires a combinatorial analysis of the declarative graph that represents the process.

This builds on previous work of the author \cite{cdm} wherein the combinatorial diversity metrics for declarative processes were derived for use in policy process analysis. The collection of stakeholder utilities can themselves then be used to form a metric with which we can compare different declarative processes to one another. These are illustrated using several examples of declarative processes that already exist in the literature.
\end{abstract}


\begin{IEEEkeywords}
Declarative process, Declarative workflow, Stakeholder, Utility function, Process comparison, Linear temporal logic
\end{IEEEkeywords}
\section{Introduction}
We present a method for calculating and analyzing stakeholder utilities of processes that arise in, but are not limited to, the social sciences.
These areas include business process analysis~\cite{vda}, healthcare workflow analysis~\cite{chesani,hildebrandt,integrated:emergency} and policy process analysis~\cite{howlett,tonythesis,fuentes,mcsfpps}.
This method is quite general and applicable to any situation in which declarative-type constraints play a part.
A declarative process $D$ is a pair consisting of a set of activities, and a list of constraints detailing how these activities may happen in relation to one another.
For such a process, anything may happen so long as it is not explicitly forbidden by the constraint set.

Declarative processes have been used as models of business and healthcare workflows by several authors~\cite{vda,mertens,vda2009}.
The notion of a declarative process is an attractive one: simply declare the constraints on activities in a system and then let the system run or evolve according to these constraints.
An execution of such a system is a (potentially infinite) listing of the activities in the order they occur, i.e. they satisfy all of the constraints that define the system. 
Such listings are called {\it{traces}}.  Two immediate concerns arise: Is there a sensible way to quantify stakeholder satisfaction for such processes?  
Is there a sensible way to compare two processes with regard to stakeholder satisfaction? 
In this paper we will take a first look at answers to these questions.
The work we present in this paper is new in that it does not necessarily build on any existing body of work in the literature. 
The closest work by other authors to what we are examining seems to be the topic of similarity measures for business process models~\cite{similarity:bpm,similarity:bpm2,similarity:bpm3}, 
however none of the material in those papers is necessarily applicable to the modelling framework that we are considering.

In considering a declarative process as a model of some system it is natural to consider how the process behaves with respect to stakeholders.
Our previous paper \cite{cdm} introduced several metrics related to the combinatorial diversity of a declarative process for use in policy process analysis.
The purpose of that paper was to derive a metric that satisfied various properties and represented, to an extent, how `free' a given declarative process was to happen.
In this paper we have a different aim in mind.
We will consider stakeholders in the declarative process and utilities for these stakeholders, and derive a measure for stakeholder utility.
This is done by focusing on a class of representatives for a declarative process (the set of unique traces) and
determining the solution to a collection of properties which we argue such a stakeholder utility function should satisfy.
Calculating the set of unique traces requires a combinatorial analysis of the declarative graph that represents the declarative process.
We then use these stakeholder utilities in order to give a method for comparing declarative processes.
These are illustrated using several examples of declarative processes that already exist in the literature.

In Section 2 we introduce a declarative process and define the set of unique traces for a declarative process. 
In Section 3 we will present an algorithm for calculating the set of unique traces and discuss how this can be optimized. 
In Section 4 we consider stakeholders in a declarative process and suppose that each stakeholder will have a preference for or against each of the unique traces of the declarative process.
We state and explain some reasonable properties that a utility function for a stakeholder should satisfy, and solving the equations that these properties imply gives 
an expression for the stakeholder utility function. 
We then illustrate these new concepts by applying them to the {\it{Patient Handler}} declarative process. 
Three different versions of the Patient Handler declarative process (Examples ~\ref{patienthandler1} and \ref{patienthandler2}) already appear in the literature.
We calculate the stakeholder utility vector, the vector of all utility functions for a given process, for several examples and discuss the results.
In Section 5 we use the stakeholder utility vector to give a method for comparing different declarative processes with respect to a prescribed set of stakeholder preferences.
This method uses $\ell_2$-norm minimization but, as we show, is robust to `noise' in the system. We illustrate this method by comparing three declarative processes in the paper.
In Section 6 we conclude with a discussion of what we have shown.

\section{Declarative Processes and Unique Traces}
Let us first introduce some standard notation and terminology related to declarative processes \cite{vda,cdm}.
Let $\Sigma$ be a set of {\it{activities}} and let $\Sigma^{*}$ be the set of all possible sequences that one can form whose entries are element of $\Sigma$. 
That is
$$\Sigma^{*} := \{\epsilon\} \cup \{ (e_1,e_2,e_3,\ldots) ~:~ e_i \in \Sigma \},$$
where $\epsilon$ denotes the empty sequence.
A {\it{trace}} is a sequence of activities $\sigma=(e_1,\ldots,e_n)\in \Sigma^{*}$.
An {\it{event}} is an occurrence of an activity in a trace.

A declarative constraint is a constraint on the activities in a process.
We require a language through which to express temporal and modal aspects of these activities, and the natural choice for this is linear temporal logic~\cite{ltlreference}.
Linear temporal logic (LTL) is an extension of propositional logic $\mathcal{L}_P$ that includes temporal modal operators 
$\mathbf{X}$ or $\circ$ (neXt), 
$\mathbf{U}$ (Until),
$\mathbf{F}$ or $\diamond$ (Finally), 
$\mathbf{G}$ or $\Box$ (Globally), 
$\mathbf{R}$ (Release), $\mathbf{W}$ (Weak until) and $\mathbf{M}$ (strong release).
As an example, given two activities $a$ and $b$ in $\Sigma$, we may wish to specify that event $b$ must happen as a response to event $a$.
In LTL one would represent this by the LTL formula $\dG (a \dimplies \dF b)$, which can be read as
``it is globally true that ($a$ occurs implies $b$ occurs at some point after $a$)''.
However, the semantics of the Declare framework~\cite{vda2007,vda2009} are easier to grasp in this respect and uses $\dresp(a,b)$ for $\dG (a \dimplies \dF b)$.
A list of some popular Declare expressions along with their LTL equivalents is given in Figure~\ref{deccondefs}.
For readability, in this paper we will consider the constraints as expressed in Declare.

We say that a trace $\sigma$ satisfies the constraint $\dresp(a,b)$ if any occurrence of $a$ in the trace will feature an occurrence of $b$ to its right. 
To represent this we write $\sigma \dmodels \dresp(a,b)$. 
It may be the case that $a$ and $b$ are not events in $\sigma$, in which case $\sigma$ certainly satisfies the constraint $\dresp(a,b)$.
If $\Const$ is a set of constraints, then we will write $\sigma \models \Const$ if $\sigma \models x$ for all $x \in \Const$.

Let us consider the trace $\sigma = (3,3,2,4,1,4)$ with $\Sigma=\{1,2,3,4,5\}$.
The trace $\sigma$ satisfies the declarative constraint $\dresp(2,1)$, i.e. $\sigma \dmodels \dresp(2,1)$ since event $1$ happens after event $2$ in $\sigma$.
However, both $\sigma \dmodels \dresp(2,3)$ and $\sigma \dmodels \dresp(2,5)$ are false.

\begin{definition}
A {\it{declarative process}} is a process on a set of activities $\Sigma$ that satisfies all conditions in a set $\Const$ of declarative constraints.
We will represent this as a pair $D=(\Sigma,\Const)$.
The set of traces of the process is
$$\Traces(D) ~=~ \{ \sigma\in \Sigma^{*} ~:~ \sigma \dmodels \Const\}.$$
\end{definition}

Restrictions on the beginning and ending of these processes may be incorporated into the constraint set using declarative constraints.

\begin{example}\label{simplefive}
Consider the declarative process $D=(\Sigma,\Const)$ where $\Sigma=\{1,2,3,4,5\}$ and
$\Const = \{ \dresp(1,2), \dprec(2,3), \dprec(3,5), \dsucc(1,4), \dnotsucc(4,2)  \}$.
Examples of traces for this process include $\epsilon$, $(2)$, $(1,2,3,5,4)$, $(2,2)$, $(2,2,\ldots,2,3)$, and $(2,2,\ldots)$.
There are will be an infinite number of traces, so $|\Traces(D)|=\infty$.
\end{example}

\begin{figure*}[!h]\label{deccondefs}
\small
\centerline{
\begin{tabular}{|l|l|l|} \hline
    Declare Constraint          & Explanation  & LTL expression \\ \hline \hline
    $\dparticipation(a)$ & Event $a$ occurs at least once & $\dF a$ \\ \hline
    $\dinitial(a)$      & Event $a$ is first to occur  & a \\ \hline
    $\dresp(a,b)$       & If event $a$ occurs, then event $b$ occurs after $a$ & $\dG (a \dimplies \dF b)$ \\ \hline
    $\dchainresp(a,b)$  & If event $a$ occurs, then event $b$ occurs & $\dG (a \dimplies \dN b)$ \\ 
						&  immediately after $a$ & \\ \hline
    $\dprec(a,b)$       & Event $b$ occurs only if preceded by event $a$ & $(\neg b) \mathbf{W} a$ \\ \hline
    $\dsucc(a,b)$       & Event $a$ occurs iff it is followed by event $b$ & $\dG (a \dimplies \dF b) ~ \wedge ((\neg b) \mathbf{W} a)$\\ \hline
    $\dnotcoexist(a,b)$    & Events $a$ and $b$ cannot coexist $b$ & $\neg (\dF a \Leftrightarrow \dF b)$ \\ \hline
\end{tabular}}
\caption{Some typical Declare constraints}
\end{figure*}

How should one approach analysing declarative processes?
In theory one could derive a measure from simply looking at the two constraint listings that define them.
A difficulty with this approach is that there are many different types of relations that can link two activities.
The interactions between these constraints are consequently far more involved than, say, those represented by directed edges in a graph and studying the resulting graph properties.

An alternative way to consider and analyse such systems is to study the set of traces, $\Traces(D)$, of the declarative process.
A drawback to this is that such a set can be infinite as in Example~\ref{simplefive}.
We might then consider finite versions of the process that contain only traces of finite length ~\cite{ltlf}, however the drawback in this case is more serious in that 
in the application areas we envisage, the occurrence of an activity (at some time perhaps far in the future) 
is more critical to our analysis than, say, a million occurrences of two activities up to the point of 
trace truncation.

This consideration leads us to considering traces in which an activity of $\Sigma$ occurs at most once in a trace.
In our first paper~\cite{cdm} on this subject we were able to justify this consideration as `first passage/time traces'. 
No choice of projection from a set of infinite objects to a set of finite objects comes without a drawback.
However we feel that the considerations of the application area combined with the notion of first passage times make this the best set of representatives for our consideration.
We note that this could be specified at the constraint level by including into the set of constraints a declarative constraint on every activity that it can not happen more than once. 
An equivalent way to consider this is to simply look at the subset of traces that contain at most once occurrence of every event.

\begin{definition}
Let $D=(\Sigma,\Const)$ be a declarative process. 
Let $\UniqueTraces(D)$ be the set of those traces $\sigma \in \Traces(D)$ for which every activity in $\sigma$ is unique.
\end{definition}

\ADdiagramfigure
\PHonediagramfigure

\begin{example}
Consider the declarative process $D=(\Sigma,\Const)$ given in Example~\ref{simplefive} where 
$\Sigma=\{1,2,3,4,5\}$ and
$\Const = \{ \dresp(1,2), \dprec(2,3), \dprec(3,5), \dsucc(1,4), \dnotsucc(4,2)  \}$.
The set $|\Traces(D)|=\infty$ while
\begin{align*}
\UniqueTraces(D) = \{ & \epsilon, (2), (2, 3), (1, 2, 4), (2, 3, 5),\\
 & (1, 2, 3, 4), (1, 2, 4, 3), (1, 2, 3, 4, 5),\\ & (1, 2, 3, 5, 4), (1, 2, 4, 3, 5)\}.
\end{align*}
\end{example}

\begin{example}[After Dinner]
\label{adk}
The following is a description of the house rules for a child between the end of dinner and going to bed.
After dinner is finished (and it must be finished) the table must be tidied.
If they want to do a jigsaw that the table must have been tidied beforehand.
The doing of a jigsaw means this jigsaw must be tidied away afterwards.
The child can watch a bedtime television show only after finishing dinner.
The child cannot get ready for bed before the jigsaw has been tidied away (in the event it needs to be).
The child cannot get ready for bed before watching the bedtime show.
The child cannot get ready for bed before tidying the table.

\PHtwodiagramfigure

To model this as a declarative process, label the events as follows:
\begin{enumerate}
\item Finish dinner.
\item Tidy table.
\item Do jigsaw.
\item Tidy away jigsaw.
\item Watch the bedtime show.
\item Get ready for bed.
\end{enumerate}

The above description translates into the following constraint set (see Figure~\ref{adonetwo} for an illustration of the constraints):
\begin{align*}
\Const_{AD1} = \{ & \dparticipation(1), \dresp(1,2), \dprec(1,5), \dprec(2,3), \\
					&  \dsucc(3,4), \dnotsucc(6,4),  \dnotsucc(6,5), \\
					&  \dnotsucc(6,2) \}.
\end{align*}
and to the declarative process $D_{AD1} = \left(\{1,2,3,4,5,6\},\Const_{AD1}\right)$.
As a declarative process, it is easy to see that $\Traces(D_{AD1})$ will have an infinite size. 
As a process, it is clear from the description that each of the activities is intended to happen at most once.
To analyse this declarative process, we are interested in $\UniqueTraces(D_{AD1})$, which is
\begin{align*}
\MoveEqLeft  \UniqueTraces(D_{AD1}) = \\
\{ & (1, 2), (1, 2, 5), (1, 5, 2), (1, 2, 6), (1, 2, 3, 4), (1, 2, 5, 6), \\
 & (1, 5, 2, 6), (1, 2, 3, 4, 5), (1, 2, 3, 5, 4), (1, 2, 5, 3, 4), \\
 & (1, 5, 2, 3, 4), (1, 2, 3, 4, 6), (1, 2, 3, 4, 5, 6), (1, 2, 3, 5, 4, 6), \\
 & (1, 2, 5, 3, 4, 6), (1, 5, 2, 3, 4, 6)\}
\end{align*}
This set reveals that, through the rules the parents laid down, the child does not necessarily have to ever get ready for bed (as evidenced by eight traces that do not contain activity 6), 
and satisfies all the rules they must follow.
If one now includes the rule that activity 6 must happen, i.e.
$$\Const_{AD2} = \Const_{AD1} \cup \{\dparticipation(6)\}$$
and
$$D_{AD2} = (\{1,2,3,4,5,6\},\Const_{AD2}),$$
then one finds those traces that the parents intended in the first place:
\begin{align*}
\MoveEqLeft \UniqueTraces(D_{AD2}) = \\
\{ & (1, 2, 6), (1, 2, 5, 6), (1, 5, 2, 6), (1, 2, 3, 4, 6), \\ 
&  (1, 2, 3, 4, 5, 6), (1, 2, 3, 5, 4, 6), (1, 2, 5, 3, 4, 6), \\ 
& (1, 5, 2, 3, 4, 6)\}.
\end{align*}
\end{example}

\section{Enumerating and generating the unique trace representatives}\label{enumeration}

In order to generate the set of unique traces of a declarative process, we have to consider all events that can occur in such a trace.
This can be any subset $X$ of the set of activities $\Sigma$.
We must then consider all the different permutations of events in $X$ to see if such a sequence satisfies the set of constraints.
Algorithm~\ref{alg:ut} gives a simple procedure for doing this. 


\begin{algorithm}
\caption{Generating the set of unique traces $\UniqueTraces(D)$}\label{alg:ut}
\begin{algorithmic}[1]
\Procedure{$\UniqueTraces$}{$\Sigma,\Const$}
\State $A \gets \emptyset$
\For{$X\subseteq \Sigma$}
	\For{$\pi \in \mathsf{Permutations}(X)$}
		\If{$\pi$ $\dmodels$ $\Const$}
			\State $ A  \gets A \cup\{\pi\}$
		\EndIf
	\EndFor
\EndFor
\State \textbf{return} $A$
\EndProcedure
\end{algorithmic}
\end{algorithm}

The above algorithm will of course be dependent upon the size of $\Sigma$ and the number of of times `satisfies' is called will be $2.718 |\Sigma| !$
For example, if $\Sigma$ has 7 activities then $2.718 |\Sigma|! = 2.718 \times 7 \times 6 \times \cdots \times 1 = 13700$.
This fact will lead to an exponential slowdown for every extra activity that is considered in $\Sigma$.
Consequently, the runtime on an average PC will increase from minutes to double digit-hours as $|\Sigma|$ goes from 9 to 15.

We can remove needless checking by stripping a declarative process down to a smaller smaller core process.
This is done by removing the equivalent of {\it{leaves}} from the constraint list (and by extension the activity set).
For example, if we have a declarative process $D=(\Sigma,\Const)$ and the only appearance of activity $j$, say, in $\Const$ is as $\dresp(i,j)$, then
we can construct $\UniqueTraces(D)$ from $\UniqueTraces(D')$ where $D'=(\Sigma\backslash\{j\},\Const\backslash \{\dresp(i,j)\}$ 
by using a simple $\dresp$ leaf addition procedure as given in Algorithm~\ref{alg:resp}.

\begin{algorithm}
\caption{Generating $\UniqueTraces(D)$ from $\UniqueTraces(D')$ where $j \not\in \Sigma'$ and $D$ contains the additional constraint $\dresp(i,j)$\label{alg:resp}}
\begin{algorithmic}[1]
\Procedure{TRLeaf}{$\UniqueTraces(\Sigma',\Const'),\dresp(i,j)$} 
\State $A \gets \emptyset$
\For{$\sigma \in \UniqueTraces(D')$}
	\If{$i \in \sigma$} 
		\For{$k \in \{\mathrm{index}(\sigma,i),\ldots,\mathrm{length}(\sigma)\}$} 
			\State $\mu \gets \sigma$ with $j$ inserted after the $k$th entry 
			\If{$\mu$ $\dmodels$ $\Const'$}
				\State $A\gets  A \cup \{\mu\}$
			\EndIf
		\EndFor
	\Else{}
		\State $A \gets  A \cup \{\sigma\}$
		\For{$k \in \{\mathrm{index}(\sigma,i),\ldots,\mathrm{length}(\sigma)\}$}
			\State $\mu \gets \sigma$ with $j$ inserted after the $k$th entry
			\If{$\mu$ $\dmodels$ $\Const'$}
				\State $A \gets A \cup \{\mu\}$
			\EndIf
		\EndFor
	\EndIf
\EndFor
\State \textbf{return} $A$
\EndProcedure
\end{algorithmic}
\end{algorithm}

This procedure is useful in that it allows us to consider decompose stakeholder satisfaction for the declarative process on a smaller process. If activity $j$ features in $G_{i'}$ for some stakeholder $S_{i'}$ then
it will be possible to re-specify $G_{i'}$ on the smaller process while conditioning on the position/absence-of activity $i$ in any traces.
It also allows us to gain some insights into the distribution of activities within the set of unique traces.
This information is useful in tracking the evolution of the unique traces and could be utilized in later work to provide bounds on certain aspects of the declarative process.

Algorithms~\ref{alg:prec} and \ref{alg:succ} deal with the declarative constraints $\dprec$ and $\dsucc$, respectively.
Similar algorithms can be given for other declarative constrains and these allow for the calculation of the unique traces of declarative processes of significantly larger size. 
We have calculated unique traces for declarative processes on 23 activities using these algorithms but have not yet had cause to consider larger systems.
It is important to mention that declarative processes having the same number of activities will not necessarily have the same run-times. 
This is especially true when it comes to declarative constraints that are more involved.
A time complexity study of different divide-and-conquer approaches would be a welcome addition to this body of work.

\begin{algorithm}
\caption{Generating $\UniqueTraces(D)$ from $\UniqueTraces(D')$ where $j \not\in \Sigma'$ and $D$ contains the additional constraint $\dprec(i,j)$\label{alg:prec}}
\begin{algorithmic}[1]
\Procedure{TPLeaf}{$\UniqueTraces(\Sigma',\Const'),\dprec(i,j)$} 
\State $A \gets \emptyset$
\For{$\sigma \in \UniqueTraces(D')$}
	\State $A \gets  A \cup \{\sigma\}$
	\If{$i \in \sigma$} 
		\For{$k \in \{\mathrm{index}(\sigma,i),\ldots,\mathrm{length}(\sigma)\}$} 
			\State $\mu \gets \sigma$ with $j$ inserted after the $k$th entry 
			\If{$\mu$ $\dmodels$ $\Const'$}
				\State $A\gets  A \cup \{\mu\}$
			\EndIf
		\EndFor
	\EndIf
\EndFor
\State \textbf{return} $A$
\EndProcedure
\end{algorithmic}
\end{algorithm}

\begin{algorithm}
\caption{Generating $\UniqueTraces(D)$ from $\UniqueTraces(D')$ where $j \not\in \Sigma'$ and $D$ contains the additional constraint $\dsucc(i,j)$\label{alg:succ}}
\begin{algorithmic}[1]
\Procedure{TSLeaf}{$\UniqueTraces(\Sigma',\Const'),\dprec(i,j)$} 
\State $A \gets \emptyset$
\For{$\sigma \in \UniqueTraces(D')$}
	\If{$i \in \sigma$} 
		\For{$k \in \{\mathrm{index}(\sigma,i),\ldots,\mathrm{length}(\sigma)\}$} 
			\State $\mu \gets \sigma$ with $j$ inserted after the $k$th entry 
			\If{$\mu$ $\dmodels$ $\Const'$}
				\State $A\gets  A \cup \{\mu\}$
			\EndIf
		\EndFor
	\Else{}
		\State $A \gets  A \cup \{\sigma\}$
	\EndIf
\EndFor
\State \textbf{return} $A$
\EndProcedure
\end{algorithmic}
\end{algorithm}


\section{Measuring stake-holder utilities}
To every declarative process $D = (\Sigma, \Const)$ we may consider a set of stake-holders $S=S(D) := \{S_1,\ldots,S_m\}$. 
These stake-holders have an interest in aspects of the declarative process such as the execution order or existence of particular activities.
As such there are some traces that are more desirable than others for these stakeholders.

For example, one stake-holder might prefer an execution of the declarative process wherein a particular activity happens (be it at the end of the process or at any stage of the process).
Another preference might be that an activity $a$ happens iff activity $b$ happens afterwards.
It might be the case that a stake-holder is happy with any one of a collection of activities happening, 
or is insistent that a particular collection of activities must all happen (at least at some point).

In stating preferences for stake-holders, we are not considering these preferences to have any dynamic implications on how the process unfolds. 
In this sense it is best to assume stake-holders preferences are private during the execution of such a process.
Our measure will compare the number of outcomes of such a process to the number of outcomes that are preferable to a given stakeholder.
We are therefore not trying to determine the best outcomes of such a process, but instead consider how well a process behaves in relation to stake-holder preferences.

A stake-holder preference will be represented by an LTL expression.
We will denote the LTL expression that corresponds to a `desirable/good' outcome for stakeholder $S_i$ by $G_i$ and define $G=(G_1,\ldots,G_m)$.
The declarative system along with the stake-holders and their (private) preferences is represented by a triple $T=(D,S,G)$ that we will call a {\it{declarative stakeholder system}}.
A more involved analysis might involve a partial ordering of preferred outcomes with scores assigned to each.
In this paper we will restrict ourselves to a binary measure of whether a given trace is `good' for a particular stake-holder.
For stake-holder $S_i$, let us define $$\GoodTraces_i (D) = \{ \sigma \in \UniqueTraces(D) ~:~ \sigma \models G_i\}.$$
This is the set of `good' outcomes of the declarative process for $S_i$.


We wish to associate a utility $u_i$ to stake-holder $S_i$ that represents their satisfaction with declarative stakeholder system $T=(D,S,G)$. 
We make the following reasoned assumptions on $u_i(T)$.
\begin{description}
\item[Assumption 1] We are not dealing with a degenerate case so $\UniqueTraces(D)$ is non-empty.
\item[Assumption 2] The utility $u_i(T)$ will be a function of both the size of $\GoodTraces_i (D)$ and $\UniqueTraces(D)$. 
\item[Assumption 3] The utility should achieve its maximum value 1 when $\GoodTraces_i (D)=\UniqueTraces(D)$ and achieve its minimum value 0  when $\GoodTraces_i (D)=\emptyset$.
\item[Assumption 4] The utility should be an increasing function of $|\GoodTraces_i (D)|$.
\item[Assumption 5] The utility should possess a scaling property so that a doubling of $|\UniqueTraces|$ does not mean that a doubling of $|\GoodTraces_i (D)|$ is required to achieve the same utility.
\end{description}

\begin{theorem}
\label{uthm}
Let $D=(\Sigma,\Const)$ be a declarative process. 
Let $S=\{S_1,\ldots,S_n\}$ be a set of stake-holders and let $G=\{G_1,\ldots,G_n\}$ be the set of preferences for the stake-holders. 
Suppose that Assumptions 1--5 hold true. 
Then the stakeholder utility vector of the declarative stakeholder system $T=(D,S,G)$ is $u(T) = (u_1(T),\ldots,u_n(T))$ where
$$u_i(T) = \dfrac{\ln (1+|\GoodTraces_i (D)|)}{ \ln (1+ |\UniqueTraces(D)|)}.$$
\end{theorem}

\begin{proof}
Let us write $a = |\GoodTraces_i (D)|$ and $b=|\UniqueTraces(D)|$ so that, by Assumptions 2 and 3,
$u_i(T) = f(a,b)$ with $f(0,b)=0$ and $f(b,b)=1$. Note that by Assumption 1 we have $b > 0$.
Assumption 4 tells us the utility should increase with the size of $\GoodTraces_i (D)$, so the function $f(x,b)$ should be an increasing function of $x$.

However, if one sets $f(x,b)$ to simply be a function $g(x/b)$ that is increasing, then we rule out being able to incorporate important aspects with relation to how such a utility should behave with respect to different scalings.
In order to accommodate Assumption 5 let us assume $f(x,b) = g(x)/g(b)$ for some function $g(x)$. 
With regard to the properties of $f$ outlined above, these imply that $g(x)$ must satisfy the following:
$$g(0)=0 \mbox{ and $g(x)$ is an increasing function of $x$}.$$

The function $g$ is not a direct measure of utility, but represents the weight attached to the number of desirable traces for stake-holder $S_i$.
Consider instances of $\GoodTraces_i (D)$ that have 0, 1, 2, and 100, valid traces (this is similar to what was considered in \cite{cdm}).
An empty $\GoodTraces_i (D)$ indicates no desirable executions of the process for user $S_i$.
If $\GoodTraces_i (D)$ consists of a single trace then it is better (for $S_i$) than the previous case of no traces.
If $\GoodTraces_i (D)$ consists of 2 traces then it is certainly better (for $S_i$) than a process that only has one trace.
However, we would consider a set of preferred traces having 101 traces to be better, but only marginally, to a process that has 100 traces.

The simplest function that represents this situation is one that is inversely proportional to its argument, i.e. satisfies the differential equation $g'(x) = k_1/(x+k_2)$.
In order for the general solution to this, $g(x) = k_1\ln(x+k_2)+c$ for constants $k,c$, to represent our situation we must have $k_1>0$.
If there is no trace in the set of desirable outcomes, then we will have $g(0) = k_1\ln(k_2) + c$. In order for this to equal 0, we must have $k_2=1$ and $c=0$ and this implies
$g(x) = k_1 \ln(x+1) $.

This now gives us the required expression for the utility function for stake-holder $S_i$:
\begin{align}
u_i(T) &= \dfrac{g(a)}{g(b)} = \dfrac{ \ln (a+1) }{\ln (b+1)} = \dfrac{\ln (1+|\GoodTraces_i (D)|)}{ \ln (1+ |\UniqueTraces(D)|)}.
\end{align}
\end{proof}


\begin{example}[After Dinner cont'd]
Consider the two declarative processes $D_{AD1}$ and $D_{AD2}$ from Example~\ref{adk}. 
Suppose that the two stake-holders are $S_1$, the child, and $S_2$, the parents and $S=\{S_1,S_2\}$.
We will consider some different forms for $G_1$ and $G_2$ and calculate their utility in order to see how the declarative processes compare for the stake-holders.
Note that $|\UniqueTraces(D_{AD1})| = 16$ and $|\UniqueTraces(D_{AD2})| = 8$.

Before doing this, let us reiterate a point made at the beginning of this section. In this study, once a preference for one stakeholder is stated, it is natural to assume that that stake-holder will 
engage in some activities to force a preferential outcome. 
This certainly might be the case and, in the case of two stake-holders having somewhat complementary preferences, it might be considered as a competition.
Our purpose is not to study how effective such stake-holders are in forcing the outcome of a process. 
(A stake-holder might not be involved in the execution of a process or have any impact on the activities of that process.)
Instead, our goal is to analyse how `good' a declarative process is in relation to stated stake-holder preferences.

\begin{enumerate}
\item[(i)] 
On a given evening, the child has a desire to watch the bedtime show after dinner. The parents are interested in the child getting ready for bed.
To assign LTL expressions to these events, we have
$$G_1 = \dparticipation(5)  \mbox{ and } G_2 = \dparticipation(6).$$
Given these expressions, we find that there are 12 traces in $\UniqueTraces(D_{AD1})$ that contain activity 5, so $|\GoodTraces_1 (D_{AD1})| = 12$.
There are 8 activities in $\UniqueTraces(D_{AD1})$ that contain activity 6, so $|\GoodTraces_2 (D_{AD1})| = 6$.
These values allow us to calculate utilities for the declarative process $AD1$:
$$u_1(D_{AD1},S,G)  = \dfrac{\ln (1+12)}{\ln(1+16)} = 0.90531$$  and  $$u_2(D_{AD1},S,G)  = \dfrac{ln (1+6)}{\ln (1+16)} = 0.68682.$$
Likewise, for the second declarative process $D_{AD2}$ we find that $|\GoodTraces_1 (D_{AD2})| = 6$ and $|\GoodTraces_2 (D_{AD1})| = 8 $, from which we calculate the utilities:
$$u_1(D_{AD2},S,G)  = \dfrac{\ln (1+6)}{\ln(1+8)} = 0.88562$$ and $$ u_2(D_{AD2},S,G)  = \dfrac{\ln (1+8)}{\ln (1+8)} = 1.$$
The first stake-holders utility is better with process AD1 whereas the opposite is true for the second stake-holder.
\item[(ii)] 
On a given evening, the child has a particular wish to do their jigsaw and then watch the bedtime show before tidying the jigsaw.
The parents, for some reason or another, would rather that the child not watch television after dinner and before going to bed.
The LTL expressions corresponding to these propositions are 
\begin{align*}
G'_1= & \dparticipation(3) \wedge \dparticipation(5) \wedge \dsucc(3,5)\\ & \wedge \dsucc(5,4)\\
G'_2 =& \neg \,  \dparticipation(5).
\end{align*}
\end{enumerate}
There are 2 traces in $\UniqueTraces(D_{AD1})$ that satisfy $G'_1$ and 4 traces in $\UniqueTraces(D_{AD1})$ that satisfy $G'_2$.
There is a single trace in $\UniqueTraces(D_{AD2})$ that satisfies $G'_1$ and 2 traces in $\UniqueTraces(D_{AD2})$ that satisfy $G'_2$.
The utilities for the two different stake-holders for both declarative processes are now:
\begin{align*}
u_1(D_{AD1},S,G') & = \dfrac{\ln (1+2)}{\ln (1+16)} = 0.38776\\  u_2(D_{AD1},S,G') &= \dfrac{\ln (1+4)}{\ln (1+16)} = 0.56806\\
u_1(D_{AD2},S,G') & = \dfrac{\ln (1+1)}{\ln (1+16)} = 0.24465\\  u_2(D_{AD2},S,G') &= \dfrac{\ln (1+2)}{\ln (1+16)} = 0.38776.
\end{align*}
For this choice of $G'$, we see that both users would therefore have a preference for process AD1 over AD2.

{\bf{Discussion:}} In this example we have considered the two declarative processes $D_{AD1}$ and $D_{AD2}$. 
We have seen how we can compare these two processes with respect to two different preferences for the two different stake-holders. 
There is no reason to limit the number of stake-holders to two. 
Our method shows that in one of these cases, the utilities calculated indicate a clear preference for one declarative process over another. 
For the other set of preferences, that is not the case.
\end{example}

\begin{example}[Patient Handler 1]\label{patienthandler1}
Let us consider an example of a declarative process from \cite{vda2009} that is illustrated in Figure~\ref{ph1}.
This is a process for handling a patient at the first aid department in a hospital with a suspected arm fracture and comprises eight activities.
The patient is initially examined by a medical professional (activity 1) and the {\sf{init}} constraint on this activity means it is the first activity that occurs in any execution of this process. 
Activity 5 `medication' shares no constraints with any other processes, however the {\sf{init}} constraint on activity 1 forbids activity 5 from occurring first.
Two constraints warrant further explanation:
\begin{itemize}
\item The 1of4 constraint indicates that at least one of the four constraints 3,4,6,8 must happen. This constraint is not conditional on other constraints or activities, and so rules out the situation whereby activity 1 happens (patient examined) followed by them being given medication (activity 5) for what is a sore arm.
We are not sure why this possibility was ruled out in the original model but will not alter it since it adds to the diversity of the underlying process.
Furthermore, since we are dealing with representatives of traces, one could argue that this is represented by the trace $(1,5,8)$ whereby the patient simply chooses not to wear a sling.
\item The `optional response' from activity 3 to activity 7 is different to the normal response in that if activity 3 happens then 7 may or may not occur afterwards. However if activity 3 does not occur then activity 7 certainly cannot occur. Of course neither 3 nor 7 need occur and this constraint is still satisfied.
\end{itemize}

The unique traces of the declarative process $D_{PH1}$ are summarized in Figure~\ref{ph1traces}.
Let us now consider some stakeholders in this process. 
Recall that stakeholders do not have to have an active role in a process but may have some clear preferences regarding observed executions of the process.

\begin{description}
\item[Stakeholder $S_1$] The first stakeholder is the patient who is presenting to the emergency department. 
This patient has a phobia of surgery and is averse to medication. The LTL proposition that models the stakeholder's preferences is
$$G_1  = \neg \dparticipation(3) ~ \wedge ~ \neg \dparticipation(5).  $$
\item[Stakeholder $S_2$] The second stakeholder is the X-ray department who are overwhelmed by the number of X-rays that are needed.
The LTL proposition that models the stakeholder's preferences is
$$G_2  =  \neg\dparticipation(2).$$
\item[Stakeholder $S_3$] The third stakeholder is the surgery department who have a very small team and do not have the resources to 
dedicate to removing casts before surgery. The LTL proposition that models the stakeholder's preferences is
\begin{align*}
G_3 =& \neg (\dparticipation(3) \wedge \dparticipation(6)) \\
& \vee ~ (\dparticipation(3) \wedge \dparticipation(6) \wedge \dsucc(3,6)).
\end{align*}
\item[Stakeholder $S_4$] The fourth stakeholder is the hospital itself. Resources are often scare and a patient who needs multiple resources can 
be costly (both in work and time) for the hospital. An avoidance of the overuse of multiple costly resources is preferred. 
An instance of this is a patient who is initially given a sling, their arm does not improve, and so an X-ray indicates a fixation is the best option. 
This fixation does not solve the problem and surgery is required, followed by rehabilitation. 
The hospital has noticed that in the past there have been several such `expensive' instances among patients who have revisited multiple times.
This process execution is represented by the unique trace $\sigma = (1,8,2,4,3,7)$.
The LTL proposition that models the stakeholder's preferences is
\begin{align*}
G_4 = & \neg (\dparticipation(1) \wedge \dsucc(1,8) \wedge \dsucc(8,2) \\ & \wedge \dsucc(2,4) \wedge \dsucc(4,3) \wedge \dsucc(3,7)).\end{align*}
\item[Stakeholder $S_5$] The fifth stakeholder is the pharmaceutical industry. In this instance it benefits when patients are prescribed medication or 
their medications are used during surgery. The LTL proposition that models the stakeholders preferences for this process is
$$G_5=  \dparticipation(3) \vee \dparticipation(5).$$
\end{description}

Examining the traces in $\UniqueTraces(D_{PH1})$ with respect to the five stakeholders we find
\begin{align*}
(\GoodTraces_1(D_{PH1}),\ldots,\GoodTraces_5(D_{PH1})) \\ = (11, 3, 389, 452, 448).
\end{align*}
This gives the following collection of utilities for the stakeholders in $T_{PH1}=(D_{PH1},S,G)$:
\begin{align*}
u_1(T_{PH1}) & = \dfrac{\ln (1+11)}{\ln (1+459)} = 0.40529\\
u_2(T_{PH1}) & = \dfrac{\ln (1+3)}{\ln (1+459)}  = 0.22610 \\
u_3(T_{PH1}) & = \dfrac{\ln (1+389)}{\ln (1+459)} = 0.97308  \\
u_4(T_{PH1}) & = \dfrac{\ln (1+452)}{\ln (1+459)} = 0.99750 \\
u_5(T_{PH1}) & = \dfrac{\ln (1+448)}{\ln (1+459)} = 0.99605.
\end{align*}
This gives the stakeholder utility vector $u(T_{PH1}) = (0.40529, 0.22610, 0.97308, 0.99750, 0.99605)$ for the declarative stakeholder system $T_{PH1} = (D_{PH1},S,G)$.

\begin{example}[Patient Handler 2]\label{patienthandler2}
In this example we will consider a modified patient handler process that we call Patient Handler 2 and is motivated by an example given in Mertens et al.~\cite{mertens}. 
The process we look at is a simplified version of the one given in \cite{mertens} since 
in that paper the authors introduced a more general declarative framework that captured aspects of a healthcare process that was not captured by the original declarative process given in \cite{vda2009}.
The concerns of \cite{mertens} are quite different to ours and our purpose in looking at that declarative process is to have a second process to compare the first process PH1 to. 

With these points in mind we will describe two simplified versions of Patient Handler 2 that we will refer to as PH2a and PH2b. The difference between these two is that PH2b contains an additional activity (activity 11 in Figure~\ref{ph2diag}) that does not feature in Patient Handler 1 but which we find interesting to include for comparison purposes.

In this modification of Patient Handler 1, we have attempted to preserve the labelling of similar activities. In Patient Handler 1 there was one activity (activity 5) for the patient being given medication. Patient Handler 2 considers a collection of different medications, and those medications that correspond to the old activity 5 have been labelled 51, 52, and 53 in order to provide a comparison between the two processes. Moreover, there are now activities for prescribing anti-inflammatory and anti-coagulation drugs after surgery. These did not feature in PH1.
Information about the unique traces of the two processes $D_{PH2a}$ and $D_{PH2b}$ can be found in Figures~\ref{ph2atraces} and \ref{ph2btraces}, respectively.

Let us consider stakeholders in these processes precisely as with did in PH1. As activity 11 (physiotherapy) does not impact on any of the preferences for the stakeholders $S_1$, $\ldots$, $S_6$, we will assume that the expressions for stakeholder preferences for PH2a and PH2b are the same. 

\begin{description}
\item[Stakeholder $S_1$] The patient who has a phobia of surgery and is averse to medication. The LTL proposition that models the stakeholder's preferences is now
\begin{align*}
G_1  =  \neg ( & \dparticipation(3) \vee \dparticipation(51) \\ & \vee \dparticipation(52) \vee \dparticipation(53)\\ & \vee \dparticipation(9) \vee \dparticipation(10) ).  
\end{align*}
\item[Stakeholder $S_2$] The overwhelmed  X-ray department.
The LTL proposition that models the stakeholder's preferences is the same as before
$$G_2  =  \neg\dparticipation(2).$$
\item[Stakeholder $S_3$] The under-staffed surgery department.
The LTL proposition that models the stakeholder's preferences is the same as before
\begin{align*}
G_3 = & \neg ( \dparticipation(3) \wedge \dparticipation(6)) \\ & \vee  (\dparticipation(3) \wedge \dparticipation(6) \wedge \dsucc(3,6)).
\end{align*}
\item[Stakeholder $S_4$] The hospital that would like to cut down on a particularly common overuse of its resources.
The LTL proposition that models the stakeholder's preferences is the same:
\begin{align*}
G_4 = \neg ( & \dparticipation(1) \wedge \dsucc(1,8)  \wedge \dsucc(8,2) \\ & \wedge \dsucc(2,4) \wedge \dsucc(4,3) \wedge \dsucc(3,7)).
\end{align*}
\item[Stakeholder $S_5$] The pharmaceutical industry that benefits when patients are prescribed or given medications.
The LTL proposition that models the stakeholder's preferences is
\begin{align*}
G_5 = & \dparticipation(3) \vee \dparticipation(51) \\ & \vee \dparticipation(52)  \vee \dparticipation(53) \\ & \vee \dparticipation(9) \vee \dparticipation(10).\end{align*}
\end{description}

Examining the traces in $\UniqueTraces(D_{PH2a})$ and $\UniqueTraces(D_{PH2b})$ with respect to the five stakeholders we find
\begin{align*}
\MoveEqLeft (\GoodTraces_1(D_{PH2a}),\ldots,\GoodTraces_5(D_{PH2a})) \\ = & (324, 1457048, 16316590, 16285678, 16316266)
\end{align*}
 and 
\begin{align*}
\MoveEqLeft (\GoodTraces_1(D_{PH2b}),\ldots,\GoodTraces_5(D_{PH2b})) \\ = & (1952,  16316590, 199143708, 198749700, 199141756).
\end{align*}
The collection of utilities for the declarative stakeholder system $T_{PH2a}=(D_{PH2a},S,G)$ is:
\begin{align*}
u_1(T_{PH2a}) & = \dfrac{\ln (1+324)}{\ln (1+16316590)} = 0.34826\\
u_2(T_{PH2a}) & = \dfrac{\ln (1+1457048)}{\ln (1+16316590)} = 0.85454\\
u_3(T_{PH2a}) & = \dfrac{\ln (1+16316590)}{\ln (1+16316590)} = 1.00000 \\
u_4(T_{PH2a}) & = \dfrac{\ln (1+16285678)}{\ln (1+16316590)} = 0.99989\\
u_5(T_{PH2a}) & = \dfrac{\ln (1+16316266)}{\ln (1+16316590)} = 0.99999.
\end{align*}
This gives the stakeholder utility vector $u(T_{PH2a}) = (0.34826,0.85454,1.00000,0.99989,0.99999)$.
The collection of utilities for the declarative stakeholder system $T_{PH2b}=(D_{PH2b},S,G)$ is:
\begin{align*}
u_1(T_{PH2b}) & = \dfrac{1+\ln (1952)}{1+\ln (199143708)} = 0.39651\\
u_2(T_{PH2b}) & = \dfrac{1+\ln (16316590)}{1+\ln (199143708)} = 0.86908\\
u_3(T_{PH2b}) & = \dfrac{1+\ln (199143708)}{1+\ln (199143708)} = 1.00000\\
u_4(T_{PH2b}) & = \dfrac{1+\ln (198749700)}{1+\ln (199143708)} = 0.99990 \\
u_5(T_{PH2b}) & = \dfrac{1+\ln (199141756)}{1+\ln (199143708)} = 0.99999
\end{align*}
This gives stakeholder utility vector $u(T_{PH2b}) = (0.39651,0.86908,1.00000,0.99990,0.99999)$.
\end{example}

\section{Comparing processes using stakeholder utility vectors}
The stakeholder utility vector of a declarative stakeholder system is a vector of 
values between 0 and 1 in which the $i$th entry represents the utility to user $i$. 
With the notation we have been using, we have the 
$$u(T) = (u_1(T),\ldots,u_m(T)).$$
The optimal outcome for all stakeholders would be for $u(T) = (1,1,\ldots,1)$.
Given a collection of declarative processes $D_1,\ldots,D_t$, in order to determine which declarative process is `optimal' with respect to all stakeholders, 
one can simply determine the stakeholder utility vector that is closest to $(1,1,\ldots,1)$ in Euclidean $m$-space.
This is done by using the Euclidean norm, also known as the $\ell_2$-norm, of a vector in $m$-space:
$$\norm{(v_1,\ldots,v_m)} ~:=~ \sqrt{ v_1^2+\ldots+v_m^2}.$$
Using this we determine which declarative stakeholder system $T_i=(D_i,S,G)$ minimizes 
$$\min_{1\leq i\leq n} \norm{u(T_i) - (1,1,\ldots,1)}.$$
Let us define $H(T_i) :=  \norm{u(T_i) - (1,1,\ldots,1)}$.

\begin{example}\label{justbefore}
Consider Examples~\ref{patienthandler1} and \ref{patienthandler2}.
We have 
\begin{align*}
\MoveEqLeft H(T_{PH1}) =\\ \MoveEqLeft \norm{ 
(0.40529, 0.2261, 0.97308, 0.9975, 0.99605)
- (1,1,1,1,1)} \\ &= 0.97640 \\
\MoveEqLeft H(T_{PH2a}) = \\ \MoveEqLeft \norm{ 
(0.34826,0.85454,1,0.99989,0.99999)
- (1,1,1,1,1)} \\ & = 0.66778 \\
\MoveEqLeft H(T_{PH2b}) = \\ \MoveEqLeft \norm{ 
(0.39651,0.86908,1,0.9999,0.99999)
- (1,1,1,1,1)}\\ & = 0.61753.
\end{align*}
The minimum of these is $H(T_{PH2b})$ and so the declarative stakeholder system $PH2b$ is the optimal choice from the set $\{PH1,PH2a,PH2b\}$.
\end{example}

The method of $\ell_2$-norm minimization is known to be sensitive to a moderate change in one of the values. 
In our setting this might amount to  a single stakeholder's utility changing dramatically.
In order to make our method more robust to such changes, we consider the optimal declarative processes 
for all possible subsets of stakeholders and make an informed decision from this based on what we observe. 

Given a subset $X=\{i_1,\ldots,i_k\}$ of stakeholders $S=(S_1,\ldots,S_m)$, let us consider the reduced stakeholder utility vector for those entries given in $X$:
$$u^{(X)} := (u_{i_1}(T),\ldots,u_{i_k}(T))$$
For every such vector $u^{(X)}$ we will consider the declarative process that minimizes the $\ell_2$-norm from it to the best case utility $(1,1,\ldots,1) \in \mathbb{R}^{|X|}$.
The result will be a list of $2^{|S|}-1$ declarative processes. 
Let $H^{(X)}(T) := \norm{u^{(X)}(T) - (1,1,\ldots,1)}$.
We then use the information given in this list to determine the optimal choice of declarative process for all stakeholders.

\begin{example}\label{finalexample}
Let us consider PH1, PH2a and PH2b in Example~\ref{justbefore}.
The table in Figure~\ref{finaltable} records the reduced stakeholder utility vectors for all possible subsets of stakeholders.
For each we record in the rightmost column the optimal choice of process.

\begin{figure*}
\begin{center}
\begin{tabular}{c|c|c} \hline
Subset & \multirow{3}{*}{$\left(H^{(X)}(T_{PH1}),~ H^{(X)}(T_{PH2a}),~ H^{(X)}(T_{PH2b})\right)$} & Process \\
$X=\{S_{i_1},\ldots,S_{i_k}\}$ & & $j\in \{PH1,PH2a, PH2b\}$ \\ 
of stakeholders && that minimizes $H(T_j)$\\ \hline \hline
$\{S_1\}$ & (0.59471, 0.65174, 0.60349) &  PH1 \\
$\{S_2\}$ & (0.77390, 0.14546, 0.13092) &  PH2b \\
$\{S_3\}$ & (0.02692, 0.00000, 0.00000) & PH2a \\
$\{S_4\}$ & (0.00250, 0.00011, 0.00010) & PH2b \\
$\{S_5\}$ & (0.00395, 0.00001, 0.00001) & PH2a \\
$\{S_1,S_2\}$ & (0.97601, 0.66778, 0.61753) & PH2b \\
$\{S_1,S_3\}$ & (0.59532, 0.65174, 0.60349) & PH1 \\
$\{S_1,S_4\}$ & (0.59472, 0.65174, 0.60349) & PH1 \\
$\{S_1,S_5\}$ & (0.59472, 0.65174, 0.60349) & PH1 \\
$\{S_2,S_3\}$ & (0.77437, 0.14546, 0.13092) & PH2b \\
$\{S_2,S_4\}$ & (0.77390, 0.14546, 0.13092) & PH2b \\
$\{S_2,S_5\}$ & (0.77391, 0.14546, 0.13092) & PH2b \\
$\{S_3,S_4\}$ & (0.02704, 0.00011, 0.00010) & PH2b \\
$\{S_3,S_5\}$ & (0.02721, 0.00001, 0.00001) & PH2a \\
$\{S_4,S_5\}$ & (0.00467, 0.00011, 0.00010) & PH2b \\
$\{S_1,S_2,S_3\}$ & (0.97638, 0.66778, 0.61753) & PH2b \\
$\{S_1,S_2,S_4\}$ & (0.97602, 0.66778, 0.61753) & PH2b \\
$\{S_1,S_2,S_5\}$ & (0.97602, 0.66778, 0.61753) & PH2b \\
$\{S_1,S_3,S_4\}$ & (0.59532, 0.65174, 0.60349) & PH1 \\
$\{S_1,S_3,S_5\}$ & (0.59533, 0.65174, 0.60349) & PH1 \\
$\{S_1,S_4,S_5\}$ & (0.59473, 0.65174, 0.60349) & PH1 \\
$\{S_2,S_3,S_4\}$ & (0.77437, 0.14546, 0.13092) & PH2b \\
$\{S_2,S_3,S_5\}$ & (0.77438, 0.14546, 0.13092) & PH2b \\
$\{S_2,S_4,S_5\}$ & (0.77391, 0.14546, 0.13092) & PH2b \\
$\{S_3,S_4,S_5\}$ & (0.02732, 0.00011, 0.00010) & PH2b \\
$\{S_1,S_2,S_3,S_4\}$ & (0.97639, 0.66778, 0.61753) & PH2b \\
$\{S_1,S_2,S_3,S_5\}$ & (0.97639, 0.66778, 0.61753) & PH2b \\
$\{S_1,S_2,S_4,S_5\}$ & (0.97602, 0.66778, 0.61753) & PH2b \\
$\{S_1,S_3,S_4,S_5\}$ & (0.59534, 0.65174, 0.60349) & PH1 \\
$\{S_2,S_3,S_4,S_5\}$ & (0.77438, 0.14546, 0.13092) & PH2b \\
$\{S_1,S_2,S_3,S_4,S_5\}$ & (0.97640, 0.66778, 0.61753) & PH2b \\ \hline
\end{tabular}
\end{center}
\caption{Table for Example~\ref{finalexample}.\label{finaltable}}
\end{figure*}

\medskip

\begin{itemize}
\item Of the $5+1=6$ subsets $X$ of $S$ having size $|X| \geq |S|-1=4$, we see that 
the optimal choice is PH2b since it appears as the answer in 5 of the 6 cases.
\item Of the $10+5+1=16$ subsets $X$ of $S$ having size $|X|\geq |S|/2$ we see that 
the optimal choice is PH2b since it appears as the answer in 12 of the 16 cases.
\item 
Of the $2^5-1=31$ non-empty subsets $X$ of $S$ we see that 
the optimal choice of strategy is PH2b since it appears as the answer in 20 of the 31 cases.
\end{itemize}
Each of these agrees with what we found in Example~\ref{justbefore} when the subset of interest $X$ is the set of all stakeholders $S$.
There appears to be no clear reason to consider that either of the other processes could be optimal for this particular collection of stakeholders.

\end{example}

The previous example highlights the reasoning and analysis that allows us to conclude that a particular declarative process is the optimal one for particular set of stakeholders.
The data in Table~\ref{finaltable} might have been different and so to address this let us make the following comments in relation to what one should do in that event:

\subsection{Choosing between optimal answers for $X=S$ and $X \subset S$}\label{careful}
Let us suppose that the optimal choice in Figure~\ref{finaltable} was not necessarily the same as the optimal choice for the other subsets of $S$ that we considered.
How should one go about deciding on an answer that can be considered robust?
There are several things to consider in this setting. In the list below the word `optimal' signifies it is optimal with respect to frequency.
Suppose that 
\newcommand{\one}{\mathsf{all}}
\newcommand{\two}{\mathsf{almostall}}
\newcommand{\three}{\mathsf{morethanhalf}}
\newcommand{\four}{\mathsf{any}}
\begin{itemize}
\item Process $T_{\one}$ is optimal for $X=S$.
\item Process $T_{\two}$ is optimal among $$\{X ~:~ X \subseteq S\mbox{ with } X \neq \emptyset \mbox{ and } |X| \geq  |S|-1\}.$$
\item Process $T_{\three}$ is optimal among $$\{X ~:~ X \subseteq S\mbox{ with } X \neq \emptyset \mbox{ and } |X| \geq  |S|/2\}.$$
\item Process $T_{\four}$ is optimal among $$\{X ~:~ X \subseteq S\mbox{ with } X \neq \emptyset\}.$$
\end{itemize}
With these notions defined we can make the following observations.
\begin{enumerate}
\item If ${\one}={\two}$ then, in the absence of any other information, the optimal choice of declarative process is clearly $D_{\one}$.
\item 
If ${\one}\neq {\two}$ then it would seem that the optimal choice of declarative process needs to be considered. 
If our motivation is to decide on a process that is strictly optimal for {\bf{all}} stakeholders then $D_{\one}$ is that process.
However process $D_{\two}$ could be considered as an alternative in the event that (a) the difference between $H(T_{\one})$ and $H(T_{\two})$ is very small 
and/or 
(b) if there is uncertainty about whether one of the stakeholders should be considered an active stakeholder at the time.
If ${\two}={\three}$ then a stronger case could be made for choosing $D_{\two}$ over $D_{\one}$, however if ${\two} \neq {\three}$ then we cannot make that case.
\item The process $D_{\four}$ might seem an odd one to consider. However it can come into play in the event that we do not know what collection of (declared) stakeholders 
are active stakeholders in some execution of a process. 
\end{enumerate}

We end this section with a brief summary of the method:
\begin{center}\fbox{\noindent
\begin{minipage}{0.9\columnwidth}
{\bf To compare the declarative stakeholder systems $T_1 = (D_1,S,G)$, $\ldots$, $T_n=(D_n,S,G)$:}\\
\hrule
\ \\[-1em]
\begin{description}[topsep=-20pt,itemsep=2pt,parsep=2pt] 
\item[Step 1] We remind ourselves that due regard must be given to stakeholder preferences in light of different declarative systems (cf. expression $G_1$ for stakeholder $S_1$ in Patient Handler 2 to the $G_1$ in Patient Handler 1).
\item[Step 2] Use Theorem~\ref{uthm} to calculate the stakeholder utility vectors $u(T_i)=(u_1(T_i),\ldots,u_m(T_i))$ for all $n$ different declarative stakeholder systems.
\item[Step 3] For each of the $2^m-1$ non-empty subsets $X$ of $\{S_1,\ldots,S_m\}$,  Determine the $j \in \{1,\ldots,n\}$ that minimizes $H^{(X)}(T_j)$ and denote this index by $\Answer(X)$.
\item[Step 4] Given the collection of optimal answers $\{ \Answer(X) ~:~ S \supseteq X \neq \emptyset \}$, analyse the four different answers 
$(T_{\one},T_{\two},T_{\three},T_{\four})$ 
one gets from Section~\ref{careful} 
in order to determine the optimal process for this collection of declarative stakeholder systems.
\end{description}
\end{minipage}
}
\end{center}

\section{Discussion}

In this paper we have presented a method for quantifying stakeholder utility for declarative processes.
This method allows us to decide which of a collection of declarative processes is optimal in light of stakeholder preferences.
As far as we are aware, there is currently no known method in the literature for performing this type of quantification and comparison.

Our method relies on constructing a set of representative traces for a declarative process followed by
considering a function on subsets of these representatives that are deemed good with respect to stakeholder preferences.
The assumptions we have used in this analysis are elementary and in subsequent work more general versions might be considered. 
For example, attributing values to traces for particular stakeholders that are different to the indicator function on traces used to determine membership of $\GoodTraces$.
Our technique will faithfully model declarative systems for which unique traces are good representatives, and we consider this to be the case with processes whose activities happen at most once during any execution. 
Many real life processes that one finds in the application areas of business process models, healthcare workflow analysis, and policy process analysis can be seen to be such systems.
Again, in subsequent work it may well be worth considering allowing activities to happen, say, at most twice but such an allowance will naturally lead to further complexity in the calculations that are necessary.

A challenge in this analysis is constructing the set of unique traces of a declarative process.
This challenging enumeration problem is a computationally demanding task since the more activities in a declarative process the more time and space this will take.
However, this can be overcome in part by a combinatorial analysis of the declarative process graph as was evidenced by the discussion of the {\it{core}} of a {\it{declarative process}} in Section~\ref{enumeration} along with 
the algorithmic considerations.
The paper lays the foundation for further investigations into this area of declarative system analysis 
and provides the technical backdrop to new research into policy process analysis (see \cite{cov19}).

\PHonetracesfigure

\PHtwoatracesfigure

\PHtwobtracesfigure


\begin{thebibliography}{99}
\bibitem{tonythesis} A.~A. Casey. A Declarative Model of the Policy Process. PhD thesis, University College Dublin, in preparation.
\bibitem{cov19} A.~A. Casey and M. Dukes. A declarative analysis of the Chinese Covid-19 emergency response. Preprint, 2021.
\bibitem{chesani} F. Chesani, P. Mello, M. Montali, and S. Storari. Testing Careflow Process Execution Conformance by Translating a Graphical Language to Computational Logic. In:  {\em Bellazzi, R., Abu-Hanna, A., Hunter, J. (eds.) AIME 2007. Lecture Notes in Computer Science (LNAI)}, vol. 4594, pp. 479--488. Springer, Heidelberg, 2007.
\bibitem{similarity:bpm} R. Dijkman, M. Dumas, B. van Dongen, R. K\"a\"arikb, and J. Mendling.  Similarity of business process models: Metrics and evaluation. {\em Information Systems} {\bf{36}} (2011), no. 2, 498--516.
\bibitem{cdm} M. Dukes and A.~A. Casey. Combinatorial diversity metrics for declarative processes: an application to policy process analysis. {\em International Journal of General Systems} {\bf{50}} (2021), no. 4, 367--387. \doi{10.1080/03081079.2021.1915305}
\bibitem{fuentes}  B.~A. Furtado, M.~A. Fuentes, and C.J.Tessone. Public Policy Modeling and Applications: state-of-the-art and perspectives. {{\em Complexity}} (special issue), 2019.
\bibitem{mcsfpps} B.~A. Furtado, P.~A.~M. Sakowski, and M.H. T\'{o}volli. {\em Modeling Complex Systems for Public Policies}.  IPEA, 2015. ISBN: 9788578112493.
\bibitem{ltlreference} D.M. Gabbay, A. Kurucz, F. Wolter, and M. Zakharyaschev. {\em Many-dimensional modal logics: theory and applications}.  Elsevier, 2003. 
\bibitem{decexample1} T. Hildebrandt and  R. Mukkamala.  Declarative Event-Based Workflow as Distributed Dynamic Condition Response Graphs. PLACES 2010. {\textit{Electronic Proceedings in Theoretical Computer Science}} {\bf{69}}:59--73, 2010.
\bibitem{hildebrandt} T. Hildebrandt, R. Mukkamala R, and  T. Slaats. Declarative Modelling and Safe Distribution of Healthcare Workflows. In: {\em Liu Z., Wassyng A. (eds) Foundations of Health Informatics Engineering and Systems. FHIES 2011. Lecture Notes in Computer Science}, vol. 7151. Springer, Berlin, Heidelberg, 2012.
\bibitem{howlett} M. Howlett, M. Ramesh, and A. Perl.  Studying Public Policy: Principles and Processes. Fourth Edition, Oxford University Press, 2020.
\bibitem{ltlf} J. Li, L. Zhang, G. Pu, M.~Y. Vardi, and J. He. LTLf satisfiability checking. \href{http://arxiv.org/abs/1403.1666}{arXiv:1403.1666}, 2014.
\bibitem{integrated:emergency} S. Mertens, F. Gailly, D. Van Sassenbroeck, and G. Poels. Integrated Declarative Process and Decision Discovery of the Emergency Care Process.  {\em Information Systems Frontiers} \doi{10.1007/s10796-020-10078-5}, 2020.
\bibitem{mertens} S. Mertens, F. Gailly, and G. Poels.  Enhancing Declarative Process Models with DMN Decision Logic. In: {\em Gaaloul K., Schmidt R., Nurcan S., Guerreiro S., Ma Q. (eds) Enterprise, Business-Process and Information Systems Modeling. BPMDS 2015, EMMSAD 2015}. {\em Lecture Notes in Business Information Processing}, vol 214. Springer, 2015. 
\bibitem{similarity:bpm2} A. Schoknecht, T. Thaler, P. Fettke, A. Oberweis, and R. Laue.  Similarity of Business Process Models -- A State-of-the-Art Analysis. ACM Computing Surveys {\bf{50}} (2017), no. 4. Article No. 52.
\bibitem{vda2007} M. Pesic, H. Schonenberg, and W. van der Aalst. DECLARE: Full support for loosely-structured processes. In {\em EDOC. IEEE Computer Society}, 287--300. \doi{10.1109/EDOC.2007.25}, 2007.
\bibitem{similarity:bpm3} C.~S. Wahyuni, K.~R. Sungkono, and R. Sarno. Novel Parallel Business Process Similarity Methods Based on Weighted-Tree Declarative Pattern Models. {\em International Journal of Intelligent Engineering and Systems} {\bf{12}}, no.6, \doi{10.22266/ijies2019.1231.23}, 2019.
\bibitem{vda} W. van der Aalst. {\it{Process Mining: Data Science in Action}}. Second edition. Springer, 2016. \doi{10.1007/978-3-662-49851-4}.
\bibitem{vda2009} W. van Der Aalst, M. Pesic, and H. Schonenberg. Declarative workflows: Balancing between flexibility and support. {\textit{Computer Science - Research and Development}} {\bf{23}}, no. 2, 99-113, 2009.
\end{thebibliography}
\end{document}